\documentclass[a4paper,11pt]{article}
\usepackage[utf8]{inputenc}
\usepackage[english]{babel}
\usepackage{expl3}
\usepackage{amsmath,amssymb,amsfonts,amsthm,mathtools,bm}
\usepackage{siunitx}
\usepackage{booktabs}
\usepackage{microtype}
\usepackage{xcolor}
\usepackage{graphicx}
\usepackage{tikz}
\usetikzlibrary{arrows.meta,positioning}
\definecolor{cobalt}{RGB}{0,71,171}
\usepackage[colorlinks,urlcolor=cobalt,citecolor=cobalt,linkcolor=cobalt,
  pdfauthor={Sebastian Schaffer and Lukas Exl}]{hyperref}
\usepackage[nameinlink,noabbrev]{cleveref}
\usepackage{enumitem}

\usepackage{fullpage}

\usepackage{authblk}

\usepackage{url}

\newtheorem{proposition}{Proposition}
\newtheorem{remark}{Remark}
\newtheorem{corollary}{Corollary}
\newcommand{\R}{\mathbb{R}}
\newcommand{\m}{\bm m}
\newcommand{\z}{\bm z}
\newcommand{\p}{\bm p}
\newcommand{\s}{\bm s}
\newcommand{\Jop}{\bm J}
\newcommand{\Rop}{\bm R}
\newcommand{\Kop}{\bm K}
\newcommand{\Eenc}{E_{\theta}}
\newcommand{\Ddec}{D_{\xi}}
\newcommand{\Genergy}{g_{\eta}}
\newcommand{\zhat}{\widehat{\bm z}}
\DeclarePairedDelimiter{\norm}{\lVert}{\rVert}

\begin{document}
\date{}
\title{An Energy-Based Conservative--Dissipative Latent Neural Evolution Operator for Magnetization Dynamics}

\author[a,b]{Sebastian Schaffer \thanks{\texttt{sebastian.schaffer@univie.ac.at}}}
\author[a,b]{Lukas Exl \thanks{\texttt{lukas.exl@univie.ac.at}}}

\affil[a]{Math. AI/ML, Wolfgang Pauli Institute, Vienna, Austria}
\affil[b]{Department of Mathematics, University of Vienna, Vienna, Austria}

\setcounter{Maxaffil}{0}
\renewcommand\Affilfont{\itshape\small}

\maketitle
\textbf{Abstract.} 
We introduce a reduced-order learning framework that couples a convolutional autoencoder to a structured latent neural ordinary differential equation. A learned scalar generates the latent dynamics through an antisymmetric operator and a positive-definite dissipative operator. The antisymmetric component permits motion along scalar level sets, while the dissipative component decreases the scalar. Adding a coercive quadratic term gives bounded continuous-time latent trajectories and approach to the scalar's critical set under the stated assumptions. The scalar provides a structural constraint rather than an approximation of physical energy. The encoder, decoder, scalar, and operators are trained jointly on short snapshot windows without time-derivative supervision, physical-energy labels, or dissipation penalties. At inference, one initial state is encoded and evolved continuously in latent space, with decoding only at requested times. We demonstrate the framework on micromagnetic magnetization trajectories from NIST $\mu$MAG Standard Problem~4 and compare it with dissipative-only dynamics and an unstructured latent neural differential equation. Repeated training runs show competitive observed-horizon accuracy, while selected longer rollouts illustrate improved retention of approximately stationary states relative to the unstructured baseline. These results motivate scalar-generated latent dynamics as an inductive bias for reduced-order prediction and other learning tasks where bounded evolution and equilibrium behavior are important.


\textbf{MSC classes.}	78M32, 78M34, 65L05, 82D40

\section{Introduction}

Computational micromagnetics is widely used for the analysis and design of magnetic devices, including permanent magnets and integrated magnetic sensors \cite{brown1963micromagnetics,suess2018topologically,bashir2012head,fischbacher2018micromagnetics,kovacs2020computational}.
Their magnetization dynamics are commonly described by variants of {the} nonlinear Landau--Lifshitz--Gilbert (LLG) equation \cite{gilbert2004phenomenological} or the Landau--Lifshitz equation
\begin{equation}
 \partial_t\m=-c_\gamma[\m\times\bm H_{\mathrm{eff}}+\alpha\,\m\times(\m\times\bm H_{\mathrm{eff}})],
 \label{eq:llg}
\end{equation}
where $\m$ is the solution trajectory, $c_\gamma>0$, $\alpha\geq0$, and $\bm H_{\mathrm{eff}}$ is the effective field. The equation combines precession and damping while evolving a spatial vector field subject to a pointwise length constraint.

Due to the stiffness of the problem, the repeated numerical solution can become a major computational bottleneck in parameter studies, optimization, and inverse design. 
This is particularly relevant when device geometries, material parameters, or excitation fields must be varied over large design spaces. 
Fast nonlinear machine-learning surrogate models, therefore, offer a promising alternative to conventional time integration by approximating the solution operator and potentially enabling substantially cheaper repeated evaluations during inverse design and optimization.

Dynamical systems often evolve in spaces whose dimension is much larger than the number of coordinates that are required to describe a restricted family of solution trajectories. This is especially true for traditional numerical solvers. A spatial discretization of a field can require numerous degrees of freedom for an accurate numerical representation, even if the solution lies near a low-dimensional nonlinear manifold. Linear trial spaces are very efficient if the states are well represented but can become inefficient for strongly nonlinear solution sets. Nonlinear manifold models based on autoencoders provide a more flexible alternative \cite{lee2020model,champion2019data,farenga2025latent,chung2026latent,xie2024latent}. An encoder assigns reduced coordinates, and a decoder reconstructs the observation. The remaining question is how those coordinates should evolve.

Earlier work has demonstrated that {machine learning} can be applied to establish surrogate models for the prediction of micromagnetic trajectories \cite{exl2020learning, exl2021prediction, schaffer2021machine, kovacs2019learning,dolui2026neural,cai2024fast}. In this work we will present another promising reduced-order model, which is motivated by energy-based models (EBMs) \cite{du2019implicit}. The effective field $\bm H_{\mathrm{eff}}$ is given by the variational derivative of the Gibbs free energy $G$,
\begin{equation}
    \bm H_{\mathrm{eff}}[\m](\bm x)
    = -\frac{1}{\mu_0 M_s}
    \frac{\delta G}{\delta \m}[\m](\bm x),
    \label{eq:effective-field}
\end{equation}
where $M_s$ is the saturation magnetization and $\mu_0$ the vacuum permeability.
Hence, the dynamics of \eqref{eq:llg} {are} driven by the scalar {functional} $G$. This poses the question whether this property can be useful for machine learning purposes.

Likelihood-based generative EBMs assign low energy to observed configurations and estimate likelihood gradients by contrasting data (positive-phase) expectations with model (negative-phase) expectations. The negative phase is often approximated with Markov-chain Monte Carlo (MCMC), including contrastive divergence \cite{hinton2002training}, persistent chains \cite{tieleman2008training}, or chains initialized from a replay buffer \cite{du2019implicit}. This is not a clustering objective, nor does every EBM use this likelihood-based setup. 
In contrast, the present method performs no density estimation, partition-function evaluation, negative-sample generation, or MCMC. 

A first idea was to implicitly learn a neural network model $\mathcal N(\m, \p)\approx G$, where $\p$ denotes conditional parameters, from snapshots of the solution trajectories by automatic differentiation (AD) and explicit time integration. While this is probably possible, it comes with some issues. First, the feature space remains high-dimensional, and the scalar gradient must be evaluated at every solver stage. Second, outer differentiation of a trajectory loss still requires mixed second derivatives. Third, and most important, such an approach would just inherit the stiffness of the original problem, leading to a high computational demand. 

Another, much more promising approach is the combination of nonlinear manifold models with a neural ordinary differential equation (neural ODE). A neural ODE parameterizes a continuous-time vector field and can learn it from trajectories \cite{chen2018neural,rubanova2019latent}. 
{Its} continuous formulation permits evaluation at requested times.
However, an unrestricted vector-valued network does not distinguish circulation from dissipation and does not supply a scalar that is constrained to decrease. This distinction can matter for uninterrupted rollouts: agreement on short windows does not by itself prevent accumulated error from carrying a trajectory into poorly represented regions.

We generate the latent vector field from the gradient of a learned scalar. An antisymmetric channel operator produces a component tangent to scalar level sets, while a symmetric positive-semidefinite channel operator produces a decreasing component. This construction is related algebraically to Hamiltonian neural networks, gradient systems, and metriplectic formulations \cite{greydanus2019hamiltonian,morrison1986paradigm,hernandez2023port}. It is deliberately less restricted than a thermodynamic model: its scalar is learned in nonunique latent coordinates and is not inferred from physical-energy labels.
{The learned latent scalar supplies a structural constraint rather than an approximation of the Gibbs energy, as discussed in Section~4.5. The scheme separates high-dimensional reconstruction from temporal integration.}

The contributions are: (i) a spatial latent-tensor encoder--decoder coupled to sitewise conservative--dissipative dynamics{,} (ii) an autonomous continuous-time dissipation identity for the learned scalar{,} (iii) a joint snapshot-window objective containing latent and decoded-rollout losses only{,} and (iv) a joint-training comparison of dissipative-only and antisymmetric--dissipative fields with quadratic, deep, and deep--quadratic {learned scalars}.

\section{Problem formulation}
\subsection{Observed dynamics and magnetization states}
Let $\mathcal X\subseteq\R^n$ be the observation space, let $\m:[0,T]\to\mathcal X$ be a trajectory, and let $\p\in\mathcal P\subseteq\R^{n_p}$ collect parameters fixed along that trajectory. We suppose that the trajectory family is approximately represented by a smooth manifold $\mathcal M\subseteq\mathcal X$ of intrinsic dimension smaller than $n$. The aim is to learn a reduced flow for that family, not to reconstruct a full governing equation throughout $\mathcal X$. Available data are snapshot sequences
\begin{equation}
 \mathcal T^{(q)}=\{(t_i^{(q)},\m_i^{(q)},\p^{(q)})\}_{i=0}^{k_q-1},\qquad
 \m_i^{(q)}=\m^{(q)}(t_i^{(q)}),
 \label{eq:data}
\end{equation}
where $q$ indexes trajectories, $t_0^{(q)}<\cdots<t_{k_q-1}^{(q)}$, and neither derivatives nor scalar-energy labels are assumed.

For a finite-difference magnetization with $N$ cells,
\begin{equation}
 \m=(\bm \rho_1,\ldots,\bm \rho_N)\in\R^{N\times3},\qquad \norm{\bm \rho_j}_2=1,
 \label{eq:mag-constraint}
\end{equation}
so the physical state space is a product of spheres.

\subsection{Nonlocal interactions and representational scope}
Computation of the nonlocal demagnetization, or stray field, is often the most difficult and computationally demanding computation for micromagnetic simulations.
In continuum notation it satisfies
\begin{equation}
 \nabla\times\bm H_{\mathrm d}=\bm0,\qquad
 \nabla\cdot(\bm H_{\mathrm d}+M_s\m\chi_\Omega)=0\quad\text{in }\R^3,
 \label{eq:magnetostatics}
\end{equation}
where $\Omega$ is the magnetic body and $\chi_\Omega$ its indicator function. Hence, $\bm H_{\mathrm d}[\m](\bm x)$ depends on magnetization throughout the sample. Its energy is
\begin{equation}
 G_{\mathrm d}(\m)=-\frac{\mu_0M_s}{2}\int_\Omega \m(\bm x)\cdot\bm H_{\mathrm d}[\m](\bm x)\,\mathrm d\bm x.
 \label{eq:demag-energy}
\end{equation}
Traditional numerical solvers need a way to compute $\bm H_{\mathrm d}[\m]$ at every step. Especially for small step sizes in explicit solvers this can quickly become an issue.
As we will see, our latent model does not evaluate this field at each ODE step but rather {learns to represent the effects of local and nonlocal interactions implicitly from trajectory data}.

\subsection{Energy decay of the reference dynamics}\label{sec:gibbs-decay}

The structure imposed in \Cref{sec:latent-dynamics} is motivated by a
property of the observed dynamics themselves. For a constant applied field, the
total Gibbs free energy is non-increasing along solutions of equation~(\ref{eq:llg}).

\begin{proposition}[Gibbs energy decay]\label{prop:gibbs}
Let $\bm{m}$ solve \eqref{eq:llg} with $\|\bm{m}\|_2 = 1$ pointwise,
$c_\gamma > 0$ and $\alpha \ge 0$, and let $G$ be the Gibbs free energy with
$\bm{H}_{\mathrm{eff}}$ given by \eqref{eq:effective-field}. If the applied field
is constant in time, then
\begin{equation}\label{eq:gibbs-decay}
  \frac{\mathrm{d}}{\mathrm{d}t} G[\bm{m}(t)]
  = -\mu_0 M_s c_\gamma \alpha
    \int_\Omega \Bigl( \|\bm{H}_{\mathrm{eff}}\|_2^2
      - (\bm{m}\cdot\bm{H}_{\mathrm{eff}})^2 \Bigr)\,\mathrm{d}x
  \;\le\; 0 .
\end{equation}
\end{proposition}

\begin{proof}
By \eqref{eq:effective-field} and the chain rule,
$\frac{\mathrm{d}}{\mathrm{d}t} G
 = \int_\Omega \frac{\delta G}{\delta \bm{m}} \cdot \partial_t \bm{m} \,\mathrm{d}x
 = -\mu_0 M_s \int_\Omega \bm{H}_{\mathrm{eff}} \cdot \partial_t \bm{m}\,\mathrm{d}x$.
Inserting {\eqref{eq:llg}}, the precessional contribution vanishes because
$\bm{H}_{\mathrm{eff}} \cdot (\bm{m}\times\bm{H}_{\mathrm{eff}}) = 0$. For the
damping contribution, the vector triple product and $\|\bm{m}\|_2 = 1$ give
$\bm{H}_{\mathrm{eff}} \cdot \bigl(\bm{m}\times(\bm{m}\times\bm{H}_{\mathrm{eff}})\bigr)
 = (\bm{m}\cdot\bm{H}_{\mathrm{eff}})^2 - \|\bm{H}_{\mathrm{eff}}\|_2^2$,
which yields \eqref{eq:gibbs-decay}. The integrand is nonnegative by the
Cauchy--Schwarz inequality, and vanishes precisely where $\bm{H}_{\mathrm{eff}}$
is parallel to $\bm{m}$, which is Brown's
equilibrium condition{.}
\end{proof}

{For $\alpha>0$, the dissipation rate vanishes only at equilibrium. For $\alpha=0$, the total energy is conserved even during precessional motion.}
Only the total energy is monotone: individual contributions need not be, since the
exchange energy typically increases while a reversal structure forms. The latent
model developed below imposes an analogous one-scalar decrease property in learned
coordinates, under the same restriction to a constant conditional parameter.

\section{Model architecture}
The full proposed reduced-order model consists of three main maps:
\begin{align}
    \Eenc:\mathcal{X} \rightarrow \mathcal{Z},\\
    F_{\eta,\omega}:[0, T] \times \mathcal{Z} \times \mathcal{P} \rightarrow \mathcal{Z},\\
    \Ddec: \mathcal{Z} \rightarrow \mathcal{X},
\end{align}
where $\mathcal{Z} = \mathbb{R}^d$ is the latent space, $(\Eenc, \Ddec)$ represents an autoencoder and $F_{\eta,\omega}$ is the latent vector field model. The encoder parameters are denoted by $\theta$, the decoder parameters by $\xi$, the energy model {parameters} by $\eta$ and the operator parameters collectively by $\omega$. The complete model parameter set is
\begin{equation}
    \Theta = (\theta, \xi, \eta, \omega).
\end{equation}
Ideally, $\Ddec\circ\Eenc$ approximates the identity on $\mathcal M$.

The inference path is shown in \Cref{fig:architecture}. An initial observation
$\m(t_0)$ is encoded as $\z_0$. The latent ODE evolves $\z_0$ to any requested
time, and the decoder maps the latent solution back to the observation space.
The physical parameter $\p$ conditions the scalar potential and hence the
latent vector field.

\begin{figure}[ht]
\centering
\resizebox{\textwidth}{!}{\begin{tikzpicture}[node distance=7mm and 8mm,>=Latex,
 block/.style={draw,rounded corners=2pt,align=center,minimum height=12mm,minimum width=24mm,fill=blue!4},
 state/.style={draw,align=center,minimum height=12mm,minimum width=25mm,fill=gray!8},
 ode/.style={draw,rounded corners=2pt,align=center,minimum height=25mm,minimum width=47mm,fill=orange!9},
 parameter/.style={draw,rounded corners=8pt,align=center,fill=green!8},arrow/.style={->,thick}]
\node[state] (m0) {$\m(t_0)$\\observed spatial state};
\node[block,right=of m0] (enc) {$\Eenc$\\encoder};
\node[state,right=of enc] (z0) {$\z_0$\\latent space \\representation};
\node[ode,right=of z0] (ode) {conditioned latent \\energy-gradient ODE\\[1mm]
$\dot\z=F_{\eta,\omega}(\z,\p)$};
\node[state,right=of ode] (zt) {$\zhat(t)$\\latent space \\trajectory};
\node[block,right=of zt] (dec) {$\Ddec$\\decoder};
\node[state,right=of dec] (mt) {$\widehat\m(t)$\\decoded output};
\node[parameter,below=9mm of ode] (p) {$\p$\\physical parameters};
\draw[arrow] (m0)--(enc); \draw[arrow] (enc)--(z0); \draw[arrow] (z0)--(ode);
\draw[arrow] (ode)--(zt); \draw[arrow] (zt)--(dec); \draw[arrow] (dec)--(mt); \draw[arrow] (p)--(ode);
\end{tikzpicture}}
\caption{Overall architecture. The encoder is evaluated at the initial state, the ODE advances the latent space representation, and the decoder supplies requested outputs. Physical parameters condition the latent energy-gradient ODE.}
\label{fig:architecture}
\end{figure}

\section{Latent energy-based dynamics}\label{sec:latent-dynamics}
\subsection{Scalar potential and gradient force}
Let the latent energy model $\Genergy:\mathcal Z\times\mathcal P\to\R$ be differentiable. Its first argument is the latent
state $\z\in\R^d$ and its second argument is the trajectory parameter
$\p\in\R^{n_p}$. We define the latent force
\begin{equation}
  \s(\z,\p)
  :=\nabla_{\z}\Genergy(\z,\p)
  \in\R^d.
  \label{eq:latent_force}
\end{equation}
The gradient is taken only with respect to $\z$ while $\p$ is fixed.

Nonlinearity in $\z$ is important for general trajectories. If $g(\z)=a^{\mathsf T}\z+b$, its gradient is constant and an autonomous model has constant velocity. If $g(\z)=\tfrac12\z^{\mathsf T}Q\z+a^{\mathsf T}\z+b$ with symmetric $Q$, its gradient is affine and the resulting vector field is affine. A deep scalar model permits both gradient and Hessian to vary with state. In \Cref{sec:results} we show that a combination of quadratic and deep scalar {models is effective for the tested configurations}.

\subsection{Conservative--dissipative split}
Let $\Jop,\Rop\in\R^{d\times d}$ satisfy
\begin{equation}
 \Jop^{\mathsf T}=-\Jop,\qquad \Rop^{\mathsf T}=\Rop,\qquad \Rop\succeq0.
 \label{eq:operator-conditions}
\end{equation}
The latent ODE in \Cref{fig:architecture} is
\begin{equation}
  \boxed{
  \dot\z(t)
  =F_{\eta,\omega}(\z(t),\p)
  =(\Jop-\Rop)\nabla_{\z}g_\eta(\z(t),\p).
  }
  \label{eq:latent-ode}
\end{equation}
Here $\dot\z=\mathrm d\z/\mathrm dt$. The operator $\Jop$ generates motion
that does not change $g_\eta$, whereas $-\Rop$ generates motion down its
gradient. The two terms use the same scalar potential but play distinct
dynamical roles.

\begin{proposition}[Autonomous latent-energy dissipation]
For fixed $\p$, a differentiable solution of \eqref{eq:latent-ode} satisfies
\begin{equation}
 \frac{\mathrm d}{\mathrm dt}\Genergy(\z(t),\p)=-\s^{\mathsf T}\Rop\s\leq0.
 \label{eq:dissipation}
\end{equation}
\end{proposition}
\begin{proof}
The chain rule and \eqref{eq:latent-ode} give
\begin{align}
 \frac{\mathrm d}{\mathrm dt}\Genergy(\z(t),\p)
 &=\s^{\mathsf T}\dot\z
 =(\s^{\mathsf T}\Jop\s-\s^{\mathsf T}\Rop\s).
\end{align}
Since $\Jop^{\mathsf T}=-\Jop$, a scalar $\s^{\mathsf T}\Jop\s$ equals its negative and vanishes. Positive semidefiniteness of $\Rop$ gives the result.
\end{proof}

The result is exact in continuous time and requires no dissipation penalty in
the training loss. It does not say that the observation-space trajectory
decreases a known physical energy. It says that the learned scalar
$g_\eta$ is a Lyapunov-like quantity for the learned autonomous latent flow.

{More generally, let $\Jop=\Jop(\z,\p,t)$ and $\Rop=\Rop(\z,\p,t)$ satisfy \eqref{eq:operator-conditions} pointwise. For fixed $\p$ and a scalar $g_\eta(\z,\p)$ without explicit time dependence, the same calculation gives
\[
 \dot\z=[\Jop(\z,\p,t)-\Rop(\z,\p,t)]\s,
 \qquad
 \frac{\mathrm d}{\mathrm dt}g_\eta(\z(t),\p)
 =-\s^{\mathsf T}\Rop(\z,\p,t)\s\leq0.
\]
No derivatives of the operators enter this identity. State-, parameter-, and time-dependent operators are therefore compatible with scalar decrease, even though explicitly time-dependent operators give a nonautonomous vector field.}

{Such state dependence is also natural in the physical equation. Define the pointwise cross-product operator $C(\m)\bm v=\m\times\bm v$. Then $C(\m)^{\mathsf T}=-C(\m)$ and, for $\norm{\m}_2=1$,
\[
 -C(\m)^2=I-\m\m^{\mathsf T}\succeq0.
\]
Using \eqref{eq:effective-field}, equation~\eqref{eq:llg} becomes
\begin{equation}\label{eq:llg-ad}
	\partial_t\m=\frac{c_\gamma}{\mu_0M_s}
	\left[C(\m)-\alpha\bigl(-C(\m)^2\bigr)\right]
	\frac{\delta G}{\delta\m}.
\end{equation}

Thus, the physical precession and damping terms themselves have a state-dependent antisymmetric--dissipative representation. This motivates the latent construction but does not identify its learned operators or scalar with their physical counterparts.}

If the parameter varies with time, $\p=\p(t)$, then
\begin{equation}
 \frac{\mathrm d}{\mathrm dt}\Genergy(\z(t),\p(t))=-\s^{\mathsf T}\Rop\s+\nabla_{\p}\Genergy(\z(t),\p(t))^{\mathsf T}\dot\p.
 \label{eq:nonautonomous-energy}
\end{equation}
{The second term on the right-hand side accounts for parameter variation and is not calibrated to physical work. Therefore, monotonicity is guaranteed only if $\p$ is constant, or if this additional term is non-positive.}

\begin{remark}
{The operators $\Jop,\Rop$ are learned but constant with respect to $\z$, $\p$, and $t$ in the present model. This is an architectural restriction rather than a requirement of the dissipation argument. Constant channel operators limit parameter count and evaluation cost, whereas predicting a dense matrix at every ODE evaluation would require an additional parameterization. Nonlinear state dependence is assigned to the learned scalar. This choice supplies a simple shared operator geometry without making the vector field linear.}

{The autoencoder can retain spatial dimensions or remove them by global pooling. In our implementation, the latent tensor retains spatial sites. The preceding derivations also apply when $\Jop$ and $\Rop$ act sitewise along the channel dimension. In this case, interactions between latent sites are introduced through the spatial convolutions of the scalar model, as described in \Cref{sec:results}. This separation is motivated by \eqref{eq:llg-ad}: the physical precession and damping operators act pointwise, while spatial interactions enter through the variational derivative of the Gibbs energy. Sitewise channel operators combined with a spatially coupled scalar model therefore provide a natural analogous construction in latent space. Operators $\Jop$ and $\Rop$ that also couple spatial sites would offer an alternative, but are not required for this construction.}
\end{remark}

{A nonlinear scalar can have several critical points and state-dependent curvature, so constant operators do not exclude multiple dynamical regimes. However, they restrict the admissible vector fields to gradients transformed by one shared operator. We do not claim that this class represents arbitrary nonlinear or multimodal dynamics. State-dependent operators combined with a low-rank ansatz are a natural extension when this restriction is insufficient. This combination could retain a compact parameterization and is deferred to future work, as discussed in Section~9.}

{For state-dependent operators, Proposition~\ref{prop:operator_spectra} applies pointwise to the instantaneous matrices. There is then no single constant operator spectrum for the entire flow. Even for constant operators, their spectra should not be identified with the spectrum of the nonlinear vector-field Jacobian, which also contains the Hessian of the learned scalar.}

\subsection{{Coercivity and bounded latent trajectories}}\label{sec:coercivity}
{The dissipation identity holds throughout the latent state space, including away from the encoded data manifold. Scalar decrease constrains the direction of motion, but does not by itself prevent a latent trajectory from escaping to infinity along an unbounded sublevel set or establish return to the encoded manifold after an encoder error. A quadratic contribution offers a simple way to control the latent norm as well as the scalar value. When it is positive definite, its growth confines the trajectory to a bounded region. Combining this contribution with a nonnegative nonlinear scalar retains this control while allowing a more flexible local geometry. To explain the resulting dynamics, we first distinguish equilibrium conditions from local stability and then establish boundedness and approach to the critical set.}

\subsubsection{{Conservative motion, equilibria, and local dynamics}}
{Write $\Kop=\Jop-\Rop$ for the combined operator.}
If $\Rop=0$, velocity is orthogonal to the energy gradient and motion is tangent to energy level sets. If $\Jop=0$, the model is a generalized gradient flow. The combined field superposes these motions. Every critical point, $\s(\z_\star,\p)=0$, is an equilibrium. The converse need not hold: a nonzero gradient can lie in the nullspace of $\Kop$.

The dissipation rate vanishes when $\Rop\s=0$. The state can still move under $\Jop\s$. Hence, zero dissipation is not equivalent to equilibrium.

{If $\Rop$ is positive definite, this degeneracy is absent. Indeed,
\[
\frac{\mathrm d}{\mathrm dt}g_\eta(\z(t),\p)
\leq-\lambda_{\min}(\Rop)\norm{\s(\z(t),\p)}_2^2.
\]
Zero dissipation then implies $\s=0$ and therefore equilibrium. Conversely, multiplying $\Kop\s=0$ by $\s^{\mathsf T}$ shows that every equilibrium is a scalar critical point. This characterizes stationary states without establishing which one a trajectory approaches. A positive-definite scalar Hessian at a critical point is sufficient for local asymptotic stability when the operators are continuous and $\Rop$ is positive definite nearby. For a scalar parameterized by an unconstrained neural network, however, positive definiteness of the Hessian is not guaranteed and must be verified at the critical point under consideration. An arbitrary critical point need not be stable, and its decoded state need not be a physical equilibrium.}

Constant matrices $\Jop$ and $\Rop$ do not make the model linear. For the vector field \eqref{eq:latent-ode} its state dependence is determined by the energy gradient only.

For a fixed parameter \(\p\), an affine energy produces a constant latent vector field. The encoder would therefore have to map each physical trajectory to straight-line motion with constant velocity in latent space. This is likely too restrictive to produce a useful latent representation of complex magnetization dynamics.

For a quadratic energy, $\nabla_{\z}g_{\mathrm{quad}}$ is affine in $\z$, and the resulting vector field is also affine.
Although still restrictive, such dynamics may be sufficient for some applications. A convex quadratic energy additionally imposes a simple energy geometry and can therefore provide a useful inductive bias.

For a deep nonlinear energy, the gradient depends nonlinearly on $\z$, allowing the vector field to vary along a trajectory. This dependence is also visible in its Jacobian,
\begin{equation}
\frac{\partial F}{\partial\z}(\z,\p) = (\Jop-\Rop)\nabla_{\z}^{2}\Genergy(\z,\p),
\end{equation}
which is state dependent whenever the energy Hessian is state dependent.

The deep--quadratic model combines both contributions:
\begin{equation}
g_{\mathrm{deep\text{-}quad}}=g_{\mathrm{quad}}+g_{\mathrm{deep}},
\end{equation}
and hence
\begin{equation}
F_{\mathrm{deep\text{-}quad}}=(\Jop-\Rop)\nabla_{\z}g_{\mathrm{quad}} + (\Jop-\Rop)\nabla_{\z}g_{\mathrm{deep}}.
\end{equation}
The quadratic term provides an affine contribution, while the deep term provides a nonlinear correction. This decomposition is similar in spirit to the combination of linear and nonlinear terms in \cite{linot2023stabilized}, but here both terms are generated by energy gradients and retain the conservative--dissipative structure.

{
\subsubsection{Boundedness and asymptotic behavior}
A positive-definite quadratic contribution provides compact sublevel sets when the remaining scalar contribution is nonnegative. The nonlinear contribution need not be convex, so this condition does not exclude multiple equilibria.

\begin{proposition}[Boundedness and global existence]\label{prop:boundedness}
Fix $\p$ and suppose
\[
 \Genergy(\z,\p)=g_0(\z,\p)+\kappa\norm{\z}_2^2,
 \qquad g_0(\z,\p)\geq0,\quad\kappa>0,
\]
where $g_0(\cdot,\p)\in C^2(\R^d)$. Let $\Jop(\z,\p)$ and $\Rop(\z,\p)$ be locally Lipschitz and satisfy \eqref{eq:operator-conditions} pointwise. Then the autonomous dynamics $\dot\z=(\Jop-\Rop)\nabla_{\z}\Genergy$ have a unique solution for every initial state $\z_0$, defined for all $t\geq0$, with
\[
 \norm{\z(t)}_2^2\leq\frac{\Genergy(\z_0,\p)}{\kappa}.
\]
\end{proposition}
\begin{proof}
The scalar is coercive since $\Genergy(\z,\p)\geq\kappa\norm{\z}_2^2$. Dissipation gives $\Genergy(\z(t),\p)\leq\Genergy(\z_0,\p)$, proving the bound and keeping the solution in a compact sublevel set.
Since $g_0(\cdot,\p)\in C^2(\mathbb{R}^d)$ and the operators
$J$ and $R$ are locally Lipschitz, the vector field is
locally Lipschitz on $\mathbb{R}^d$. By the
Picard--Lindelöf theorem, there exists a unique maximal
solution $\z:[0,T_{\max})\to\mathbb{R}^d$.
The dissipation identity and coercivity imply
\[
  \kappa\|\z(t)\|_2^2
  \le g_\eta(\z(t),\p)
  \le g_\eta(\z_0,\p),
  \qquad 0\le t<T_{\max}.
\]
Hence, the trajectory remains in the compact sublevel set
\[
  K_0=\{\z\in\mathbb{R}^d:
    g_\eta(\z,\p)\le g_\eta(\z_0,\p)\}.
\]
Since the vector field is defined and locally Lipschitz
on all of $\mathbb{R}^d$, the ODE continuation theorem
excludes a finite maximal existence time while the
solution remains in $K_0$. Therefore,
$T_{\max}=\infty$.
\end{proof}

\begin{corollary}[LaSalle {\cite{lasalle1960some}}]\label{cor:lasalle}
Under the preceding assumptions, let $\mathcal K_0=\{\z:\Genergy(\z,\p)\leq\Genergy(\z_0,\p)\}$. The trajectory approaches the largest invariant subset of
\[
 \{\z\in\mathcal K_0:
 \s(\z,\p)^{\mathsf T}\Rop(\z,\p)\s(\z,\p)=0\}.
\]
If $\Rop$ is positive definite throughout $\mathcal K_0$, this invariant subset consists precisely of the critical points of $\Genergy$ in $\mathcal K_0$, all of which are equilibria.
\end{corollary}
\begin{proof}
The sublevel set is compact and forward invariant, so LaSalle's invariance principle applies. Positive definiteness of $\Rop$ makes zero dissipation equivalent to $\s=0$, and every such critical point is an equilibrium.
\end{proof}

If $\Rop$ is positive definite throughout $\mathcal K_0$,
continuity and compactness imply the existence of
$\beta_0>0$ such that
$\Rop(\z,\p)\succeq\beta_0 I$ for all $\z\in \mathcal K_0$.
Consequently,
\[
  \frac{d}{dt}g_\eta(\z(t),\p)
  \le -\beta_0\|\nabla_z g_\eta(\z(t),\p)\|_2^2.
\]
Integration and the nonnegativity of $g_\eta$ yield
\[
  \int_0^\infty
  \|\nabla_z g_\eta(\z(t),\p)\|_2^2\,dt
  \le \frac{g_\eta(\z_0,\p)}{\beta_0}.
\]
Thus, the squared scalar-gradient norm is integrable
over the entire forward trajectory. Together with
boundedness and the regularity of the vector field,
this is consistent with the asymptotic approach to
the critical set established by LaSalle's invariance
principle.

Together, these results explain the role of a coercive scalar: scalar decrease becomes a bound on the latent trajectory, and positive-definite dissipation forces approach to the critical set. In particular, $\norm{\nabla_{\z}\Genergy(\z(t),\p)}_2\to0$. This conclusion holds for arbitrary initial latent states, including states displaced from the encoded data manifold. It does not imply contraction, return to that manifold, or convergence to a single critical point. Nor does the dissipation estimate supply a uniform exponential convergence rate without additional curvature assumptions. A small minimum eigenvalue of $\Rop$ permits weak dissipation in some directions. The quadratic term therefore supplies a useful boundedness guarantee while leaving predictive accuracy and the physical meaning of decoded equilibria to be assessed separately.

\begin{remark}[Bounded decoded trajectories]
Proposition~3 also implies boundedness of the decoded
continuous-time trajectory. Indeed, for any continuous
decoder $D_\xi:\mathbb{R}^d\to\mathcal{X}$, the image
$D_\xi(K_0)$ of the compact invariant sublevel set is
compact. Consequently,
\[
  \sup_{t\ge 0}\|D_\xi(\z(t))\| < \infty.
\]
However, this does not imply preservation of the pointwise
unit-magnetization constraint or convergence to a
physically admissible equilibrium.
\end{remark}
}

\subsection{Spectral properties of the operators}
The structural assumptions also constrain the spectra of the constant
operators. Although these spectra do not by themselves determine the nonlinear
dynamics, they are useful for verifying a trained model and further model analysis.

\begin{proposition}[Operator spectra]
Let $\Jop,\Rop\in\R^{d\times d}$ satisfy
$\Jop^{\mathsf T}=-\Jop$ and $\Rop\succeq0$, and define
$\Kop=\Jop-\Rop$. Then:
\begin{enumerate}[label=(\roman*)]
  \item every eigenvalue of $\Jop$ is purely imaginary or zero{,}
  \item every eigenvalue of $\Rop$ is real and nonnegative{,} and
  \item every eigenvalue $\lambda$ of $\Kop$ satisfies
        $\operatorname{Re}(\lambda)\leq0$.
\end{enumerate}
\label{prop:operator_spectra}
\end{proposition}

\begin{proof}
The first two statements are standard consequences of real antisymmetry and
real symmetric positive semidefiniteness. For the third, let
$\bm v\in\mathbb C^d\setminus\{\bm0\}$ satisfy
$\Kop\bm v=\lambda\bm v$, and let $\bm v^\ast$ denote its conjugate transpose.
Then
\begin{equation}
  \lambda
  =\frac{\bm v^\ast\Jop\bm v}{\bm v^\ast\bm v}
   -\frac{\bm v^\ast\Rop\bm v}{\bm v^\ast\bm v}.
\end{equation}
The first quotient is purely imaginary, while the second is real and
nonnegative. Therefore
\begin{equation}
  \operatorname{Re}(\lambda)
  =-\frac{\bm v^\ast\Rop\bm v}{\bm v^\ast\bm v}
  \leq0.
\end{equation}
\end{proof}
{The eigendecompositions of $\Jop$ and $\Rop$ characterize the channel operators and can inform a potential low-rank ansatz. Their spectra alone do not quantify the contributions of the operators along a trajectory. They do not remove the coordinate and potential-scale ambiguities discussed later.}

\subsection{Potential offset and scale}
For any parameter-dependent scalar $b(\p)$,
\begin{equation}\widetilde g(\z,\p)=\Genergy(\z,\p)+b(\p)\quad\Longrightarrow\quad \nabla_{\z}\widetilde g=\nabla_{\z}\Genergy.\end{equation}
Thus, the energy offset is unidentifiable. For any $a>0$,
\begin{equation}
 \widetilde g=a\Genergy,\qquad \widetilde\Jop=a^{-1}\Jop,\qquad \widetilde\Rop=a^{-1}\Rop
 \label{eq:scale-ambiguity}
\end{equation}
{leaves the vector field unchanged. Trajectory supervision therefore cannot fix an absolute latent-scalar or operator scale. Together with the offset freedom and the nonuniqueness of the learned coordinates, this prevents a direct comparison of absolute scalar values between independently trained models or random seeds. The useful object for trajectory prediction is the resulting vector field rather than a physically calibrated value of $g_\eta$.}

{The transformation in \eqref{eq:scale-ambiguity} does not change the exact predicted trajectories, so fixing this scale is not required to define the surrogate flow. A normalization could select a convenient representation convention for the scalar and operators, but would neither align independently learned coordinates nor provide physical calibration. Dynamically equivalent parameterizations can nevertheless have different optimization conditioning. We have not studied normalization or convergence to equivalent representations across initializations and do not infer that this ambiguity has no effect on training.}

\begin{remark}
{The learned scalar is not physically calibrated. Its role is structural: its gradient generates the latent dynamics, while its non-increasing value imposes a Lyapunov-like inductive bias. The operators separate motion tangent to scalar level sets from dissipative motion across them. This restriction does not by itself establish predictive accuracy. Even agreement with physical-energy values at sampled states would not necessarily establish agreement of scalar gradients or physical effective fields. The present trajectory-only training therefore supports a structural interpretation of the scalar, not its use as a physical energy.}
\end{remark}

\section{Learning from snapshot windows}
\subsection{Short windows extracted from long trajectories}

Let $\mathcal T$ denote the dataset of $M$ complete observed trajectories. During training, each trajectory is decomposed into contiguous windows of $k$ snapshots, yielding the windowed dataset
\begin{equation}
\mathcal T_k = \left\{\{(t_{\ell+i}^{(q)}, \m_{\ell+i}^{(q)}, \p^{(q)})\}_{i=0}^{k-1}\;:\;q=1,\ldots,M,\;\ell=0,\ldots,k_q-k
\right\}.
\label{eq:windowed-dataset}
\end{equation}

Windows may overlap, but complete physical trajectories are assigned to the training, validation, and test sets before window extraction. Hence, snapshots from the same trajectory cannot appear in different data partitions.

Each element of $\mathcal T_k$ is treated as an independent training sample. For a single window, the trajectory and window indices are omitted, and we write
\begin{equation}
\{\left(t_i,\m_i,\p\right)\}_{i=0}^{k-1}.
\end{equation}

The corresponding snapshots are encoded and give the latent reference states 
\begin{equation}
\z_i = \Eenc(\m_i), \qquad i=0,\ldots,k-1.
\end{equation}

The latent ODE~\eqref{eq:latent-ode} is integrated over $[t_0,t_{k-1}]$, with initial condition $\zhat(t_0) = \z_0$, and evaluated at the observation times $\zhat_i = \zhat(t_i)$ for $i=0,\ldots,k-1.$

The predicted latent states are compared to the actual latent reference states. Further, they are passed to the decoder which allows for comparison of the prediction in $\mathcal X$.
The numerical solver may use arbitrary internal time steps and only the solution values at the observation times are required for loss evaluation.

\subsection{Latent trajectory loss}
The latent trajectory loss compares the ODE solution with the encoded
observation sequence:
\begin{equation}
  \mathcal L_z
  =\frac{1}{(k-1)d}
   \sum_{i=1}^{k-1}
   \norm{\zhat_i-\z_i}_2^2.
  \label{eq:latent_loss}
\end{equation}
The initial index is omitted because $\zhat_0=\z_0$ by construction. This loss
encourages the encoder to produce coordinates whose evolution can be described
by the structured latent ODE. Because the targets $\z_i$ also depend on the
encoder parameters, $\mathcal L_z$ is coordinate dependent and can change
under a rescaling of the latent space.

\subsection{Decoded rollout loss and joint objective}
Accuracy in latent coordinates does not automatically imply accuracy after
decoding. We therefore compare the decoded integrated states directly with the
observations:
\begin{equation}
  \mathcal L_{\mathrm{roll}}
  =\frac{1}{3 k N}
   \sum_{i=0}^{k-1}
   \norm{\Ddec(\zhat_i)-\m_i}_2^2.
  \label{eq:rollout_loss}
\end{equation}
This term exposes the decoder to latent states reached by numerical
integration, not only to exact encoder outputs. It also anchors training to the
observable state rather than solely to a non-identifiable latent coordinate.
{We retain the initial index in the rollout loss. Since $\zhat_0=\Eenc(\m_0)$, its contribution is exactly the reconstruction loss of the first state. It anchors the jointly trained encoder and decoder directly to the observation space before any latent time integration is performed.}

The full objective is given by
\begin{equation}
 \mathcal L(\Theta)=\lambda_z\mathcal L_z+\lambda_{\mathrm{roll}}\mathcal L_{\mathrm{roll}},\qquad \lambda_z,\lambda_{\mathrm{roll}}>0.
 \label{eq:objective}
\end{equation}
All parameters are optimized jointly. Gradients pass through the encoder targets, ODE solution, scalar gradient, operators, and decoder. The two losses counter simple degeneracies: the decoded term anchors the model to observations, while the latent term directly compares coordinate trajectories.

\section{Differentiation, integration, and scaling}

\subsection{Neural ODE viewpoint}
A neural ordinary differential equation specifies a parameterized vector field
$F_\Theta$ and defines predictions as the solution of an initial-value problem
\cite{chen2018neural},
\begin{equation}
  \dot\z(t)=F_\Theta(t,\z(t),\p),
  \qquad
  \z(t_0)=\z_0,
  \label{eq:generic_neural_ode}
\end{equation}
or, equivalently,
\begin{equation}
  \z(t)=\z_0+\int_{t_0}^{t}F_\Theta(\tau,\z(\tau),\p)\,\mathrm d\tau.
\end{equation}
Unlike a network with a fixed number of discrete layers, the requested output
times are handled by a numerical ODE solver, which may choose its own internal
steps. Training differentiates a loss through this numerical solution map. In
the present model, the vector field is $F_{\eta,\omega}$ from
\eqref{eq:latent-ode}, rather than an unrestricted neural network. The larger
training parameter set $\Theta=(\theta,\xi,\eta,\omega)$ additionally includes
the encoder and decoder parameters, although they are not part of the latent
right-hand side itself. While a generic neural ODE is certainly also a viable model for latent space integration, it imposes a less restrictive inductive bias compared to our {energy-based} approach.

\subsection{Reverse-over-reverse differentiation}
At one right-hand-side evaluation, an inner reverse-mode derivative computes $\s=\nabla_{\z}\Genergy(\z,\p)$ and the operator produces $F=(\Jop - \Rop)\s$. Let $q_F$ be an incoming outer reverse-mode cotangent, defined by $\delta\mathcal L=q_F^{\mathsf T}\delta F$. Reverse differentiation through the operator gives
\begin{equation}
 q_s=(\Jop - \Rop)^{\mathsf T}q_F.
 \label{eq:gradient-cotangent}
\end{equation}
The outer derivative must propagate this sensitivity through the inner gradient in two directions. Its state contribution is the Hessian--vector product
\begin{equation}
 \left(\frac{\partial F}{\partial\z}\right)^{\mathsf T}q_F=\bm H_g(\z,\p)q_s,
 \label{eq:hvp}
\end{equation}
and its scalar-network parameter contribution is the mixed product
\begin{equation}
 \left(\frac{\partial F}{\partial\eta}\right)^{\mathsf T}q_F=
 \left(\frac{\partial}{\partial\eta}\nabla_{\z}\Genergy(\z,\p)\right)^{\mathsf T}q_s.
 \label{eq:mixed-product}
\end{equation}
Automatic differentiation evaluates these products without constructing a dense Hessian or mixed-derivative matrix. The calculation is reverse-over-reverse: an inner reverse pass constructs the field and an outer pass differentiates the loss through that construction and the numerical solve. 
Differentiating through adaptive solvers also requires care: continuous adjoints and direct discrete differentiation can produce different numerical gradients at finite tolerance \cite{gholami2019anode}.

\subsection{Continuous and discrete dissipation}
\eqref{eq:dissipation} concerns the exact continuous solution. A generic finite-tolerance integrator need not preserve monotonicity at every accepted step or saved output. A small numerical increase is distinct from violation of the analytic identity, which cannot occur when the operators satisfy
\eqref{eq:operator-conditions}. Solver selection and tolerances should therefore be reported as evaluation choices, not presented as part of the continuous-time guarantee.

\subsection{Computational scaling}
Let $C_g$ denote the cost of evaluating the scalar network and its gradient. One field evaluation costs approximately $C_g+\mathcal O(d^2)$ and operator storage is $\mathcal O(d^2)$ for dense operators. Structured rank-$r$ factorizations could reduce these contributions to $\mathcal O(dr)$ work and $\mathcal O(dr)$ storage if $r\ll d$. Antisymmetry and positive semidefiniteness can be retained through suitable factorizations, but the resulting rank restrictions may reduce expressivity and are not studied here. The decoder is evaluated at requested outputs rather than every internal ODE stage. This separation does not by itself establish a speed advantage: a meaningful comparison must include encoding, integration, decoding, and accuracy.

\begin{remark}
    {If the latent space retains spatial dimensions, the computational cost is $C_g+\mathcal O(|\mathcal S|n_c^2)$, where $|\mathcal S|$ denotes the number of retained spatial sites and $n_c$ is the channel width.}
\end{remark}

\section{Relation to structured dynamical models}
For canonical Hamiltonian dynamics, a fixed antisymmetric symplectic matrix acts on a Hamiltonian gradient. The antisymmetric term here has the same algebraic energy-preserving property, but the learned latent coordinates are not assumed canonical and the channel operator is learned. Pure gradient systems take the form $\dot\z=-M\nabla G(\z)$ with $M\succeq0$ and decrease $G$. The present field adds an antisymmetric component without changing the one-scalar dissipation identity.

{Related learned models use canonical Hamiltonian latent embeddings \cite{goyal2025deep} or port-Hamiltonian conservative--dissipative structure \cite{roth2026stable}.} Metriplectic and GENERIC formulations combine antisymmetric and symmetric brackets \cite{morrison1986paradigm,grmela1997dynamics}. Complete GENERIC formulations generally distinguish energy and entropy generators and impose additional degeneracy conditions. The one-scalar construction used here is simpler and should not be identified with those thermodynamic formalisms. Compared with an unrestricted latent neural ODE $\dot\z=f_\psi(\z,\p)$, the structured field occupies a smaller hypothesis class. It trades unrestricted vector-field flexibility for an exact continuous-time scalar-decrease property and explicit operator roles{. Whether} that trade is beneficial requires empirical comparison.

{Reduced-coordinate prediction of micromagnetic trajectories has already been explored with low-rank kernel methods \cite{exl2021prediction} and with convolutional autoencoders coupled to a second network for latent temporal evolution \cite{kovacs2019learning}. Further comparisons of machine-learning methods are given in \cite{schaffer2021machine}. These studies partly share the present benchmark setting: in particular, \cite{exl2021prediction,schaffer2021machine} use the same $100\times25\times1$ SP4 discretization and overlapping field ranges. The distinction is therefore not simply one of unrelated data. The low-rank kernel method predicts the next reduced state from several preceding states and reports improved stability from this multi-step history \cite{exl2021prediction}. The models in \cite{schaffer2021machine} likewise use multiple preceding states, whereas the present autonomous latent ODE starts from one encoded initial state. Short-window supervision here does not supply a history of reference states at inference.}

{The present contribution is the restriction of the continuous-time vector field to a learned scalar gradient acted on by antisymmetric and positive semidefinite operators. In contrast to a generic latent Neural ODE \cite{chen2018neural,rubanova2019latent}, this construction imposes the scalar-decrease identity by parameterization. Incorporating history through an augmented latent state or a history-dependent initialization could also be explored in this framework, but would change its information requirements and architecture. Any benefit remains to be tested.}

{A controlled quantitative comparison with the earlier micromagnetic surrogates must account for history length and initialization as well as the data, model capacity, training procedure, and rollout protocol. Such a comparison has not been performed here, and the present study does not establish superiority over those approaches.}

\section{Numerical results}\label{sec:results}
We use the two applied-field directions of NIST $\mu$MAG Standard Problem~4 \cite{nist_sp4}: field~1 at $170^\circ$ and field~2 at $190^\circ$. Each comprises 200 uniformly spaced amplitudes: \SIrange{20.00}{29.95}{mT} for field~1 and \SIrange{30.00}{39.95}{mT} for field~2, at \SI{0.05}{mT} increments. Within these ranges, {field~1} exhibits comparatively smooth and regular dynamics, whereas {field~2} displays more complex trajectory variation.

The nominal trajectories at \SI{25}{mT} and \SI{36}{mT}, respectively, are held out as reference cases. The remaining 199 trajectories are split at trajectory level into 139 training trajectories $\mathcal{T}_{\mathrm{train}}$, 20 validation trajectories $\mathcal{T}_{\mathrm{val}}$, and 40 test trajectories $\mathcal{T}_{\mathrm{test}}$.

All trajectories begin from the same relaxed S-state and contain 101 snapshots from 0 to \SI{1}{ns}, inclusive, at \SI{0.01}{ns} intervals, on a $100\times25\times1$ grid with three magnetization components. Field amplitude is scaled to $[-1,1]$ using bounds fitted on the training set. This experiment concerns the stated grid and field directions only.

{The reference trajectories were generated with our custom JAX finite-difference micromagnetic solver, available with the source code \cite{schaffer_2026_22281254}. The sample dimensions are $500\times125\times3\,\mathrm{nm}^3$, giving cells of size $5\times5\times3\,\mathrm{nm}^3$. We use saturation magnetization $M_s=8.0\times10^5\,\mathrm{A\,m^{-1}}$, exchange stiffness $A=1.3\times10^{-11}\,\mathrm{J\,m^{-1}}$, Gilbert damping $\alpha=0.02$, and gyromagnetic constant $\gamma=2.211\times10^5\,\mathrm{m\,A^{-1}\,s^{-1}}$, with $c_\gamma=\gamma/(1+\alpha^2)$ in \eqref{eq:llg}. Magnetocrystalline anisotropy is zero. Exchange is evaluated by centered finite differences with homogeneous Neumann boundary conditions. The demagnetizing field is computed by a zero-padded FFT convolution with the Newell cuboid demagnetization tensor, representing an isolated sample.}

{The common S-state is prepared by zero-field relaxation for \SI{1}{ns} with $\alpha=1$. The seed points along $+x$ in the interior $x$-slices and along $+y$ in the two outer $x$-slices. The field-on trajectories start from the relaxed state with damping restored to $0.02$. The uniform applied field is specified by $\mu_0\bm H_{\mathrm{ext}}=h(\cos\phi,\sin\phi,0)$, where $h$ is the stated amplitude in mT and $\phi$ is measured counter-clockwise from $+x$. Both relaxation and field-on evolution use Diffrax Tsit5 with adaptive PID step-size control, relative tolerance $10^{-7}$, absolute tolerance $10^{-9}$, and an initial physical step of $10^{-5}\,\mathrm{ns}$. The maximum number of adaptive steps per solve is 10,000. Internally, the solver uses reduced field $\bm H/M_s$ and time $\tau=t\gamma M_s/(1+\alpha^2)$, with $t$ expressed in seconds. Data generation uses double precision. The \SI{0.01}{ns} snapshot interval specifies saved output times rather than a fixed integration step.}

\subsection{Implementation and optimization}
The implemented encoder has channel widths $(16,32,64,128)$ and a matched transpose-convolution decoder. Its retained latent state has shape $4\times1\times1\times128$: four spatial sites with 128 channels each. All spatially nontrivial autoencoder and energy convolutions use \texttt{VALID} padding. Where a spatial extent is one, a numerically equivalent $1\times1\times1$ convolution is used when a larger \texttt{VALID} kernel is inapplicable. Hidden activations are GELUs.

{The resulting dimension $d=512$ is a fixed architectural choice balancing spatial compression against the cost of latent evolution. This configuration performed reasonably well during preliminary testing and was retained for the model comparison, rather than selected by a systematic hyperparameter study. Retaining four spatial sites permits cross-site coupling while reducing the 7,500 magnetization components. The dense channel operators have storage and evaluation costs quadratic in the channel width, which motivates limiting that width. The same representation is used for both field directions and all sampled amplitudes. However, the latent dimension alone does not determine reconstruction quality. The capacity of the encoder and decoder, controlled by their depth, channel widths, and parameter count, also determines how well the autoencoder can fit or memorize the training states at a given bottleneck size. A systematic hyperparameter study should therefore vary both the latent dimension and the encoder--decoder capacity, and assess held-out reconstruction and rollout accuracy to distinguish fitting the training data from generalization. We do not claim that 512 is minimal or optimal, or that it is sufficient for a broader family of fields and geometries. Such a joint study is deferred to future work with richer data.}

\begin{table}[ht]
\centering\small
\caption{Implemented autoencoder. Counts include biases where present.}
\label{tab:architecture}
\begin{tabular}{llr}\toprule
Stage & output tensor & parameters\\\midrule
Input & $100\times25\times1\times3$ & 0\\
Encoder block 1 & $49\times11\times1\times16$ & 1,488\\
Encoder block 2 & $23\times4\times1\times32$ & 8,768\\
Encoder block 3 & $10\times1\times1\times64$ & 34,944\\
Encoder block 4 & $4\times1\times1\times128$ & 57,600\\
Latent $1\times1\times1$ convolution & $4\times1\times1\times128$ & 16,384\\
Decoder block 1 & $10\times1\times1\times128$ & 82,176\\
Decoder block 2 & $23\times4\times1\times64$ & 139,456\\
Decoder block 3 & $49\times11\times1\times32$ & 34,912\\
Decoder block 4 & $100\times25\times1\times16$ & 8,752\\
Reconstruction $1\times1\times1$ convolution & $100\times25\times1\times3$ & 48\\
\midrule Autoencoder total && 384,528\\\bottomrule
\end{tabular}
\end{table}

{We compare three learned-scalar families. The quadratic model is
\begin{equation}
 g_{\mathrm{quad}}(\z,\p)=\tfrac12\norm{Q(\z)-c(\p)}_F^2+10^{-3}\norm{\z}_F^2,
\end{equation}
where $Q$ is a bias-free VALID spatial convolution with factor width $r=128$, and $c(\p)=W_c\p+b_c$ is a learned affine conditioning map broadcast over the two retained output sites. Thus, $Q(\z)\in\mathbb R^{2\times1\times1\times128}$ and $c(\p)\in\mathbb R^{128}$. The additional squared-norm term makes the full quadratic form positive definite even though the convolutional factor need not be injective. The deep scalar concatenates the conditioning parameters with the latent channels, applies a spatial convolution followed by pointwise convolutions of widths $(128,1)$ with GELU hidden activations, sums the output field, and applies softplus. The deep--quadratic scalar is the additive combination
\begin{equation}
 g_{\mathrm{deep\text{-}quad}}(\z,\p)=g_{\mathrm{deep}}(\z,\p)+g_{\mathrm{quad}}(\z,\p).
\end{equation}
The quadratic contribution supplies an affine term to the latent vector field, while the deep contribution permits state-dependent curvature. This construction does not impose low-rank channel operators.}

{
The coefficient of the additional squared-norm term
is fixed at $\kappa=10^{-3}$ and has not been
systematically optimized. Since the quadratic factor
is a squared norm and the deep scalar is nonnegative
by construction, both the quadratic and deep--quadratic
models satisfy the coercivity assumption of
Proposition~3. Their scalar potentials are smooth,
and the constant channel operators act sitewise,
preserving the required antisymmetry and positive
semidefiniteness. Consequently, for every fixed
conditioning parameter and arbitrary initial latent
state, the corresponding exact continuous-time
trajectories exist globally forward in time and remain
bounded. Since the implemented dissipative operator
is positive definite, Corollary~1 additionally
guarantees asymptotic approach to the critical set
of the learned scalar. Neither result guarantees
convergence to a physically meaningful equilibrium.
The deep-only model does not possess the same
coercivity guarantee.

}

The spatial convolutions in the {scalar} models are structurally important because the channel operators act sitewise and do not themselves move information between retained sites. Both $Q$ in the quadratic model and the first layer of the deep model use a spatial kernel that mixes neighboring latent sites{. The} subsequent $1\times1\times1$ layers mix channels locally. Consequently, the gradient at one retained site can depend on neighboring latent states. The convolutions are therefore the model's direct route for representing cross-site coupling, including transport-like propagation or redistribution patterns in the latent dynamics.

{Each scalar is paired with two vector-field families: dissipative-only (D), $\dot\z_a=-\Rop\nabla_{\z_a}g$, and antisymmetric--dissipative (A--D), $\dot\z_a=(\Jop-\Rop)\nabla_{\z_a}g$. Dense $128\times128$ channel operators act at each of the four latent sites: $\Jop=A-A^{\mathsf T}$ and $\Rop=B^{\mathsf T}B+\beta I$. The trainable arrays $A,B$ are unconstrained, while the construction enforces skew symmetry and positive-definite dissipation. The D and A--D operators add 16,385 and 32,769 trainable parameters, respectively.}

{The unstructured baseline uses $\dot\z=f_\psi(\z,\p)$: a SAME-padded spatial convolution with 128 hidden channels on the concatenated latent state and broadcast conditioning, GELU, and a pointwise affine map back to 128 latent channels. All models use the same autoencoder, losses, conditioning, data splits, and solver settings.}

{The architectures are not exactly parameter matched: $A$ and $B$ are structural matrix factors rather than interchangeable neural-network features. Table~\ref{tab:model-parameters} reports all parameter counts. We do not interpret this comparison as isolating structure from capacity.}

\begin{table}[ht]\centering\scriptsize
\caption{{Model parameter counts, common to both fields. The latent-scalar column excludes the operator parameters. The full vector-field column includes the scalar and the trainable $A,B,\widetilde\beta$ where applicable. Totals additionally include the 384,528 autoencoder parameters.}}\label{tab:model-parameters}
{\resizebox{\textwidth}{!}{\begin{tabular}{lrrr}
\toprule
Model & latent scalar parameters & vector field parameters & total parameters \\
\midrule
A--D quadratic & 49,408 & 82,177 & 466,705 \\
A--D deep & 66,305 & 99,074 & 483,602 \\
A--D deep--quadratic & 115,713 & 148,482 & 533,010 \\
D quadratic & 49,408 & 65,793 & 450,321 \\
D deep & 66,305 & 82,690 & 467,218 \\
D deep--quadratic & 115,713 & 132,098 & 516,626 \\
Unstructured baseline & -- & 66,176 & 450,704 \\
\bottomrule
\end{tabular}
}}
\end{table}

{Positivity is enforced by $\beta=\operatorname{softplus}(\widetilde\beta)=\log(1+\exp(\widetilde\beta))$, with an unconstrained trainable scalar $\widetilde\beta$. It is initialized as $\widetilde\beta=\log(\exp(0.01)-1)$, giving $\beta=0.01$. Thus, the parameterization makes $\Rop$ positive definite in exact arithmetic. Separately, the deep scalar also uses softplus: $g_{\mathrm{deep}}(\z,\p)=\operatorname{softplus}(u_\eta(\z,\p))$, where $u_\eta$ is the summed scalar field from the convolutional network described above. This ensures positive output values, but does not enforce convexity or strict local minima.}

{Convolution and dense kernels use LeCun-normal initialization with zero biases unless specified otherwise. For the quadratic factor, the LeCun-normal variance is multiplied by $0.05$. The conditioning map uses fan-in truncated-normal variance scaling $10^{-4}$. The entries of $A$ use fan-in truncated-normal variance scaling $0.1^2$, while $B$ is initialized as $\sqrt{0.05}$ times an orthogonal matrix. Consequently, initially $B^{\mathsf T}B=0.05I$ and $\Rop=0.06I$, including the initial $\beta=0.01$. This choice sets the initial dissipative scale. Different operator initializations can change optimization and the resulting learned dynamics, and their sensitivity has not been systematically assessed here.}

{We jointly train with window size $k=5$, $\lambda_z=\lambda_{\mathrm{roll}}=1$, batch size 32, and 1,200 epochs. AdamW uses weight decay $10^{-4}$. The learning rate warms from $5\times10^{-5}$ to $5\times10^{-4}$ during the first 5\% of updates, then decays cosinusoidally to $5\times10^{-6}$. Checkpoints and representative figures are selected by validation loss, not by test error.}

{The window length is held fixed and was not systematically optimized. Window length is coupled to the optimization horizon, the number of intermediate states decoded during training, and hence the computational budget. Its effect may also depend on the operator initialization. Moreover, a meaningful selection criterion should include uninterrupted validation rollouts, which the present training procedure does not evaluate. We therefore do not interpret a short fixed-window sweep under the unchanged procedure as a sufficient optimization study.}

{For optimization, a promising direction is to start with $k=2$ windows, containing one transition, and gradually increase the window length. Both the latent trajectory loss and the decoded rollout loss would then constrain the representation from the beginning, while the prediction horizon would grow progressively. A systematic study should compare such curricula with fixed windows, vary the operator initialization, use full-rollout validation, and assess whether every intermediate state must be decoded given the associated training cost. These questions are difficult to separate from simultaneous representation learning in the autoencoder. We therefore defer them to a dedicated study on a simpler dataset without an autoencoder, where the scalar-based Neural ODE optimization can be assessed in isolation.}

Our implementation uses Diffrax \cite{kidger2021neural} with the explicit
Tsitouras $5(4)$ Runge--Kutta method (\texttt{Tsit5}). The relative tolerance is set to $10^{-3}$ and absolute tolerance to $10^{-5}$ for training and evaluation. The adaptive solver returns states at the snapshot times used by the loss even though its internal stages generally occur at different times.

Gradients are computed with Diffrax's \texttt{RecursiveCheckpointAdjoint}. This method differentiates the discrete solver computation by reverse-mode AD while retaining only selected primal states and recomputing others during the backward pass. Recursive checkpointing therefore trades additional computation for lower memory use. It should not be confused with a continuous backsolve adjoint that integrates a separate adjoint ODE backward in time. The checkpointed discrete adjoint differentiates the numerical trajectory actually produced by the chosen solver and tolerances. Checkpointing a discrete solve trades recomputation for storage, but it does not remove the products in \eqref{eq:hvp} and \eqref{eq:mixed-product}.

{Evaluation uses one encoded initial state, no re-encoding during the rollout, and no unit-norm projection of decoded predictions. Indices $q$, $i$, and $j$ denote trajectories, times, and spatial cells, respectively. Write $\widehat{\m}_i^{(q)}$ and $\zhat_i^{(q)}$ for the decoded and latent predictions, $\z_i^{(q)}=\Eenc(\m_i^{(q)})$ for the encoded reference, and $\m_{\mathrm{rec},i}^{(q)}=\Ddec(\Eenc(\m_i^{(q)}))$ for reconstruction without time integration. For $N$ cells, define
\begin{align*}
 \mathrm{RMSE}_m^{(q)}(t_i)
 &=\left[\frac{1}{N}\sum_j
 \norm{\widehat{\m}_{i,j}^{(q)}-\m_{i,j}^{(q)}}_2^2\right]^{1/2},\\
 \mathrm{RMSE}_{\mathrm{rec}}^{(q)}(t_i)
 &=\left[\frac{1}{N}\sum_j
 \norm{\m_{\mathrm{rec},i,j}^{(q)}-\m_{i,j}^{(q)}}_2^2\right]^{1/2}.
\end{align*}
For each magnetization error type $a\in\{m,\mathrm{rec}\}$, pooling over times or trajectories retains the root-mean-square convention. The pooled error for one trajectory is
\[
 \mathrm{RMSE}_a^{(q)}
 =\left[\frac{1}{N_t}\sum_i
 \bigl[\mathrm{RMSE}_a^{(q)}(t_i)\bigr]^2\right]^{1/2}.
\]
For a dataset $\mathcal D$, such as $\mathcal T_{\mathrm{test}}$, containing $Q$ trajectories at common saved times, the pooled time-resolved error and the error pooled over both time and trajectories are
\begin{align*}
 \mathrm{RMSE}_a^{\mathcal D}(t_i)
 &=\left[\frac{1}{Q}\sum_{q=1}^{Q}
 \bigl[\mathrm{RMSE}_a^{(q)}(t_i)\bigr]^2\right]^{1/2},\\
 \mathrm{RMSE}_a^{\mathcal D}
 &=\left[\frac{1}{QN_t}\sum_{q,i}
 \bigl[\mathrm{RMSE}_a^{(q)}(t_i)\bigr]^2\right]^{1/2}.
\end{align*}
Magnetization and reconstruction RMSE equal their relative counterparts because the reference magnetization has unit length. For windowed datasets, $q$ indexes windows, $N_t=k$, and $t_i$ is measured from the start of each window.}

{For the latent error, we report the relative RMSE as the square root of the ratio of the pooled squared prediction error to the pooled squared reference magnitude,
\begin{align*}
 \mathrm{rRMSE}_z^{\mathcal D}(t_i)
 &=\left[\frac{\sum_q\norm{\zhat_i^{(q)}-\z_i^{(q)}}_2^2}
 {\sum_q\norm{\z_i^{(q)}}_2^2}\right]^{1/2},\\
 \mathrm{rRMSE}_z^{\mathcal D}
 &=\left[\frac{\sum_{q,i}\norm{\zhat_i^{(q)}-\z_i^{(q)}}_2^2}
 {\sum_{q,i}\norm{\z_i^{(q)}}_2^2}\right]^{1/2}.
\end{align*}
Taking the ratio of pooled squared norms avoids assigning excessive weight to individual latent states with small norm.}

{The pointwise unit-norm error and its dataset-wide mean and maximum are
\[
 \delta_{i,j}^{(q)}
 =\left|\norm{\widehat{\m}_{i,j}^{(q)}}_2-1\right|,
 \qquad
 \delta^{\mathcal D}
  =\frac{1}{QN_tN}\sum_{q,i,j}\delta_{i,j}^{(q)},
  \qquad \delta_{\max}^{\mathcal D}=\max_{q,i,j}\delta_{i,j}^{(q)}.
\]
We also report the mean angular error in degrees,
\[
 \mathrm{MAnE}^{\mathcal D}=\frac{180}{\pi QN_tN}
 \sum_{q,i,j}\arccos\!\left[
 \frac{\widehat{\m}_{i,j}^{(q)}\cdot\m_{i,j}^{(q)}}
 {\norm{\widehat{\m}_{i,j}^{(q)}}_2}
 \right].
\]
Here the reference magnetization has unit length and predicted vectors are assumed nonzero. All sums cover the relevant trajectories, saved times, and cells. Predicted vectors are normalized only inside the angular metric. The other errors use the raw predictions.}

The primary objective is to assess the overall model design and its applicability to this problem rather than to maximize predictive performance. Because the space of hyperparameters, architectural choices, and training algorithms is large, we use fixed configurations and do not perform systematic hyperparameter optimization. The reported results should therefore not be interpreted as estimates of the best attainable performance, and further tuning may yield substantial improvements.

\subsection{{Joint vector-field and learned-scalar comparison}}
{Three independent initializations are used for each configuration. The field-2 D-deep configuration has only two completed runs because one training job exceeded the 16-hour job limit. Tables~\ref{tab:field1-comparison} and \ref{tab:field2-comparison} pool 40 held-out trajectories per field separately for each trained model, then report the mean and sample standard deviation across the available initializations. These statistics describe a limited sample of optimization variability, not confidence intervals or evidence of a general model ranking.}

{On field~1, the A--D deep--quadratic and unstructured models have essentially equal full-rollout accuracy: $0.0371\pm0.00621$ and $0.0367\pm0.00250$, respectively. On field~2, the corresponding values are $0.129\pm0.00617$ and $0.189\pm0.0169$. Thus, the generic baseline is competitive on the observed horizon, especially for field~1. The larger field-2 errors remain important even for the best structured configuration.}

{Reconstruction errors are lower than the uninterrupted-rollout errors, indicating an additional dynamical-prediction contribution beyond compression. The encoder and decoder are smooth maps, but smoothness alone does not establish a well-conditioned latent representation. Reconstruction errors assess the jointly trained autoencoder without time integration, rather than an independently trained autoencoder. Their comparison with rollout errors does not uniquely separate representation, vector-field, and optimization limitations.}

{Small average norm errors do not establish pointwise constraint preservation. For the main A--D deep--quadratic configuration, the seed-averaged mean deviations are $9.49\times10^{-4}$ and $3.05\times10^{-3}$ for fields~1 and~2, but the corresponding averages of per-seed global maxima are $0.104$ and $0.915$. The decoder therefore has localized, substantially larger norm defects, particularly for field~2. A possible extension is a Cayley-transform output layer: the decoder would parameterize a skew-symmetric matrix at each cell, whose Cayley transform rotates a unit reference vector and hence preserves its norm. This would change the decoder parameterization and would need to be assessed during training. It is not used in the reported results.}

{The timing comparison uses the same custom finite-difference solver as data generation, but in single precision with Tsit5 tolerances $10^{-3}$ and $10^{-5}$, matching the latent ODE settings. It measures evolution from the initial to the final time: surrogate timing includes initial encoding, latent integration, and final-state decoding only. Intermediate decoding, host-to-device transfer, diagnostics, and file output are excluded. After JIT compilation and warm-up, ten repeated timings are recorded per trajectory on an NVIDIA A40 GPU, with JAX 0.11.1 and Diffrax 0.7.2. Averaging the recorded test-trajectory reference timings across the three principal-model evaluations gives \SI{0.493}{s} for field~1 and \SI{0.444}{s} for field~2, with 1,084 and 970 adaptive steps per trajectory, respectively. These are online final-state timings, not costs of double-precision data generation or relaxation.}

{Tables~\ref{tab:efficiency-field1} and \ref{tab:efficiency-field2} report training costs and solver work as well as online inference time. The baseline avoids the inner reverse-mode derivative required to evaluate a scalar-generated vector field and is substantially faster than the deep scalar models in this comparison. These online timings do not establish the total computational saving after data generation and training, whose amortization depends on how often the surrogate is reused.}

{Training costs are reported, but an end-to-end data-generation cost is not. We therefore do not give a complete break-even estimate, which would also depend on the required accuracy and output schedule.}

\begin{table}[ht]\centering\scriptsize
\caption{{Field~1 errors on 40 held-out trajectories ($170^\circ$). Entries are seed mean $\pm$ sample standard deviation after pooling within each seed. Both numbers are multiplied by the power of ten shown in the column heading. Angular errors are in degrees. A--D denotes antisymmetric--dissipative and D dissipative-only. $\delta$ is the mean unit-norm error and $\delta_{\max}$ the maximum over cells, times, and test trajectories for each seed. Bold marks the lowest mean in each column. Latent errors depend on learned coordinates and are not physical errors.}}\label{tab:field1-comparison}
{\resizebox{\textwidth}{!}{\begin{tabular}{lrrrrrrr}
\toprule
Model & \shortstack{$\mathrm{RMSE}_m^{\mathcal{T}_{\mathrm{test}}^k}$\\$(10^{-3})$} & \shortstack{$\mathrm{RMSE}_m^{\mathcal{T}_{\mathrm{test}}}$\\$(10^{-2})$} & \shortstack{$\mathrm{rRMSE}_z^{\mathcal{T}_{\mathrm{test}}}$\\$(10^{-1})$} & \shortstack{$\mathrm{RMSE}_{\mathrm{rec}}^{\mathcal{T}_{\mathrm{test}}}$\\$(10^{-3})$} & $\mathrm{MAnE}$ ($^\circ$) & \shortstack{$\delta^{\mathcal{T}_{\mathrm{test}}}$\\$(10^{-3})$} & \shortstack{$\delta_{\max}^{\mathcal{T}_{\mathrm{test}}}$\\$(10^{-1})$} \\
\midrule
A--D quadratic & $6.99 \pm 0.207$ & $32.0 \pm 0.494$ & $5.44 \pm 0.0986$ & $7.86 \pm 0.273$ & $13.1 \pm 0.252$ & $\bm{0.933 \pm 0.00809}$ & $2.98 \pm 0.318$ \\
A--D deep & $6.87 \pm 0.340$ & $11.7 \pm 5.28$ & $2.42 \pm 0.908$ & $7.77 \pm 1.30$ & $3.59 \pm 1.83$ & $1.26 \pm 0.199$ & $3.15 \pm 0.586$ \\
A--D deep--quadratic & $\bm{6.63 \pm 0.0486}$ & $3.71 \pm 0.621$ & $0.993 \pm 0.128$ & $\bm{7.05 \pm 0.0621}$ & $\bm{1.19 \pm 0.138}$ & $0.949 \pm 0.0103$ & $1.04 \pm 0.292$ \\
D quadratic & $9.07 \pm 0.213$ & $92.8 \pm 6.10$ & $12.6 \pm 1.29$ & $13.7 \pm 0.379$ & $47.2 \pm 3.80$ & $11.9 \pm 1.31$ & $4.08 \pm 1.19$ \\
D deep & $8.35 \pm 0.506$ & $124 \pm 19.7$ & $33.4 \pm 8.00$ & $12.9 \pm 1.79$ & $53.8 \pm 3.57$ & $151 \pm 95.1$ & $62.1 \pm 13.7$ \\
D deep--quadratic & $8.25 \pm 0.101$ & $92.7 \pm 0.275$ & $11.6 \pm 1.68$ & $12.2 \pm 0.103$ & $43.8 \pm 2.66$ & $22.9 \pm 11.0$ & $18.0 \pm 11.4$ \\
Unstructured baseline & $6.99 \pm 0.243$ & $\bm{3.67 \pm 0.250}$ & $\bm{0.501 \pm 0.0571}$ & $7.57 \pm 0.319$ & $1.36 \pm 0.0633$ & $0.983 \pm 0.0468$ & $\bm{0.884 \pm 0.373}$ \\
\bottomrule
\end{tabular}
}}\end{table}
\begin{table}[ht]\centering\scriptsize
\caption{{As \Cref{tab:field1-comparison}, for field~2 ($190^\circ$). D-deep statistics use two completed runs because one job exceeded the 16-hour training limit. The remaining configurations use three runs.}}\label{tab:field2-comparison}
{\resizebox{\textwidth}{!}{\begin{tabular}{lrrrrrrr}
\toprule
Model & \shortstack{$\mathrm{RMSE}_m^{\mathcal{T}_{\mathrm{test}}^k}$\\$(10^{-3})$} & \shortstack{$\mathrm{RMSE}_m^{\mathcal{T}_{\mathrm{test}}}$\\$(10^{-2})$} & \shortstack{$\mathrm{rRMSE}_z^{\mathcal{T}_{\mathrm{test}}}$\\$(10^{-1})$} & \shortstack{$\mathrm{RMSE}_{\mathrm{rec}}^{\mathcal{T}_{\mathrm{test}}}$\\$(10^{-3})$} & $\mathrm{MAnE}$ ($^\circ$) & \shortstack{$\delta^{\mathcal{T}_{\mathrm{test}}}$\\$(10^{-3})$} & \shortstack{$\delta_{\max}^{\mathcal{T}_{\mathrm{test}}}$\\$(10^{-1})$} \\
\midrule
A--D quadratic & $40.5 \pm 0.522$ & $39.2 \pm 1.28$ & $5.36 \pm 0.229$ & $42.3 \pm 0.666$ & $16.1 \pm 0.575$ & $3.37 \pm 0.204$ & $\bm{8.26 \pm 1.77}$ \\
A--D deep & $41.4 \pm 0.134$ & $36.1 \pm 1.25$ & $7.52 \pm 1.22$ & $43.8 \pm 0.199$ & $12.5 \pm 0.701$ & $8.02 \pm 1.36$ & $12.4 \pm 2.31$ \\
A--D deep--quadratic & $\bm{40.3 \pm 0.520}$ & $\bm{12.9 \pm 0.617}$ & $\bm{1.89 \pm 0.282}$ & $\bm{42.0 \pm 0.586}$ & $\bm{4.09 \pm 0.247}$ & $\bm{3.05 \pm 0.0212}$ & $9.15 \pm 0.201$ \\
D quadratic & $44.0 \pm 0.551$ & $90.6 \pm 28.2$ & $11.5 \pm 3.93$ & $48.9 \pm 0.499$ & $47.4 \pm 20.9$ & $16.5 \pm 3.86$ & $8.86 \pm 1.68$ \\
D deep & $45.7 \pm 0.147$ & $83.5 \pm 27.4$ & $17.1 \pm 0.321$ & $50.9 \pm 3.25$ & $33.3 \pm 10.3$ & $138 \pm 101$ & $34.4 \pm 19.1$ \\
D deep--quadratic & $43.1 \pm 0.183$ & $73.9 \pm 7.57$ & $6.98 \pm 3.90$ & $48.6 \pm 0.284$ & $33.0 \pm 4.43$ & $15.1 \pm 1.05$ & $9.78 \pm 0.107$ \\
Unstructured baseline & $41.1 \pm 0.334$ & $18.9 \pm 1.69$ & $2.43 \pm 0.536$ & $42.7 \pm 0.369$ & $6.60 \pm 0.554$ & $3.58 \pm 0.286$ & $9.27 \pm 0.261$ \\
\bottomrule
\end{tabular}
}}\end{table}

\begin{table}[ht]\centering\scriptsize
\caption{{Field~1 computational costs: seed mean $\pm$ sample standard deviation. Seconds per trajectory cover initial encoding, integration to \SI{1}{ns}, and final-state decoding, excluding intermediate outputs. Reference-solver final-state evolution takes \SI{0.493}{s} per trajectory on the same NVIDIA A40, in single precision at tolerances $10^{-3}/10^{-5}$. Training time excludes evaluation. Steps include accepted and rejected steps.}}\label{tab:efficiency-field1}
{\resizebox{\textwidth}{!}{\begin{tabular}{lrrrr}
\toprule
Model & training (h) & seconds/trajectory & steps & rejected (\%) \\
\midrule
A--D quadratic & $\bm{6.22 \pm 0.494}$ & $(2.80 \pm 0.82) \times 10^{-2}$ & $35.1 \pm 0.465$ & $\bm{0.00 \pm 0.00}$ \\
A--D deep & $7.40 \pm 0.322$ & $(6.73 \pm 2.10) \times 10^{-2}$ & $44.1 \pm 1.11$ & $12.0 \pm 2.71$ \\
A--D deep--quadratic & $9.45 \pm 2.19$ & $(7.38 \pm 0.27) \times 10^{-2}$ & $53.1 \pm 0.549$ & $9.39 \pm 0.666$ \\
D quadratic & $7.22 \pm 0.120$ & $(4.50 \pm 0.41) \times 10^{-2}$ & $86.2 \pm 6.62$ & $19.9 \pm 0.833$ \\
D deep & $7.35 \pm 0.421$ & $(3.75 \pm 0.96) \times 10^{-2}$ & $\bm{34.0 \pm 9.14}$ & $16.8 \pm 0.541$ \\
D deep--quadratic & $11.4 \pm 1.03$ & $(1.34 \pm 0.16) \times 10^{-1}$ & $95.4 \pm 6.47$ & $20.0 \pm 0.277$ \\
Unstructured baseline & $9.79 \pm 0.616$ & $\bm{(2.76 \pm 0.15) \times 10^{-2}}$ & $47.6 \pm 0.600$ & $12.8 \pm 1.25$ \\
\bottomrule
\end{tabular}
}}\end{table}
\begin{table}[ht]\centering\scriptsize
\caption{{As \Cref{tab:efficiency-field1}, for field~2. The reference final-state timing is \SI{0.444}{s} per trajectory. D-deep uses two completed runs. Other entries use three.}}\label{tab:efficiency-field2}
{\resizebox{\textwidth}{!}{\begin{tabular}{lrrrr}
\toprule
Model & training (h) & seconds/trajectory & steps & rejected (\%) \\
\midrule
A--D quadratic & $6.39 \pm 0.389$ & $\bm{(2.45 \pm 0.17) \times 10^{-2}}$ & $44.3 \pm 1.37$ & $\bm{2.17 \pm 0.713}$ \\
A--D deep & $7.05 \pm 0.158$ & $(5.18 \pm 0.30) \times 10^{-2}$ & $44.1 \pm 2.51$ & $15.5 \pm 0.234$ \\
A--D deep--quadratic & $10.0 \pm 1.05$ & $(9.50 \pm 0.37) \times 10^{-2}$ & $64.2 \pm 2.20$ & $10.2 \pm 1.11$ \\
D quadratic & $\bm{6.15 \pm 0.167}$ & $(2.85 \pm 0.32) \times 10^{-2}$ & $54.3 \pm 6.05$ & $18.3 \pm 1.28$ \\
D deep & $11.3 \pm 5.86$ & $(5.65 \pm 1.36) \times 10^{-2}$ & $\bm{25.9 \pm 6.72}$ & $16.1 \pm 0.202$ \\
D deep--quadratic & $12.3 \pm 1.06$ & $(1.27 \pm 0.14) \times 10^{-1}$ & $81.0 \pm 7.87$ & $17.8 \pm 1.57$ \\
Unstructured baseline & $8.15 \pm 0.885$ & $(2.78 \pm 0.06) \times 10^{-2}$ & $49.8 \pm 1.13$ & $14.7 \pm 0.696$ \\
\bottomrule
\end{tabular}
}}\end{table}

\begin{figure}[ht]
\centering\includegraphics[width=.92\textwidth]{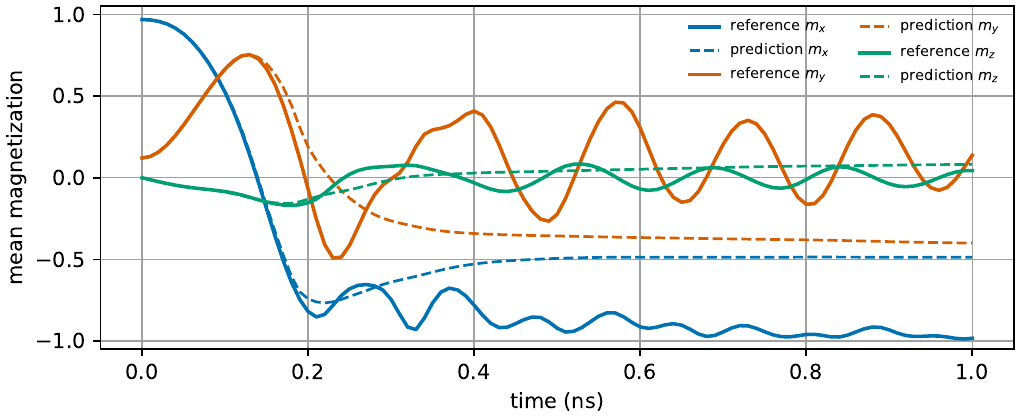}
\caption{{Uninterrupted dissipative-only deep--quadratic rollout at \SI{25}{mT}, $170^\circ$, on the $100\times25\times1$ grid. Solid curves are reference spatial means and dashed curves are decoded spatial means averaged over three trained models.}}
\label{fig:dissipative-failure-field1}
\end{figure}

\begin{figure}[ht]
\centering\includegraphics[width=.92\textwidth]{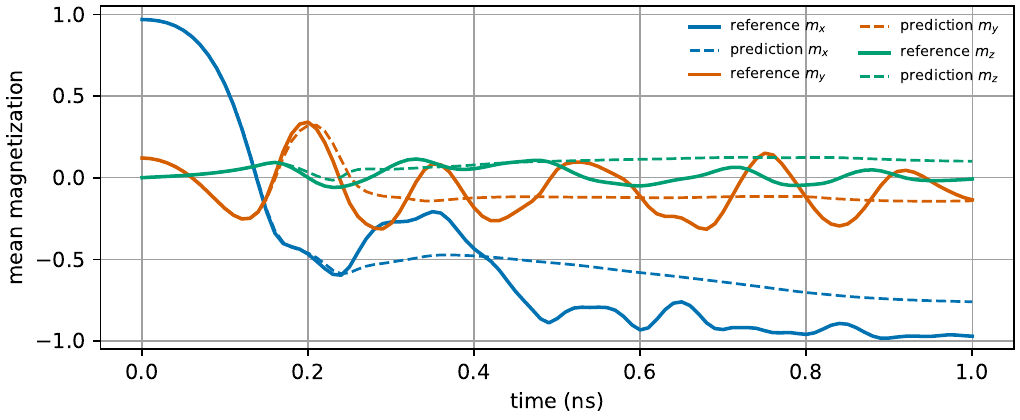}
\caption{{As \Cref{fig:dissipative-failure-field1}, at \SI{36}{mT}, $190^\circ$.}}
\label{fig:dissipative-failure-field2}
\end{figure}

It can be seen that a small error on the windowed dataset $\mathcal{T}_\mathrm{test}^k$ does not {provide} any strong guarantees for longer rollouts.
The dissipative-only models have much smaller window errors than their uninterrupted-rollout errors, but do not reproduce the full rollouts in these two comparisons. This supports the usefulness of the antisymmetric term for these data and configurations. 

{\Cref{fig:dissipative-failure-field1,fig:dissipative-failure-field2} show the uninterrupted dissipative-only deep--quadratic rollouts for the two reference fields. The comparison of predicted and reference mean magnetization illustrates how errors accumulate over a complete trajectory even when short-window errors remain small. For these configurations, adding the antisymmetric component improves uninterrupted-rollout accuracy. This does not establish that an antisymmetric operator is universally required or that the structured model outperforms an unrestricted latent Neural ODE.}

{The quadratic-only scalar also gives larger uninterrupted-rollout errors than the deep--quadratic model despite small window errors.}

\subsection{{Amplitude dependence and time-resolved errors}}
{Figures~\ref{fig:amplitude-field1} and \ref{fig:amplitude-field2} compare the A--D deep--quadratic model with the unstructured baseline across held-out amplitudes. Each curve is the mean across three trained models, with a one-standard-deviation band. Magnetization errors pool squared errors over saved times. Relative latent errors divide the pooled squared prediction error by the pooled squared reference magnitude before taking a square root. The field-2 dependence is nonmonotone, with larger errors at the low-amplitude edge as well as a pronounced increase toward the high-amplitude end. The latter is consistent with the more intricate spatial and temporal magnetization evolution in that regime. This is an observed association within the sampled interval.}
\begin{figure}[ht]\centering\includegraphics[width=.92\textwidth]{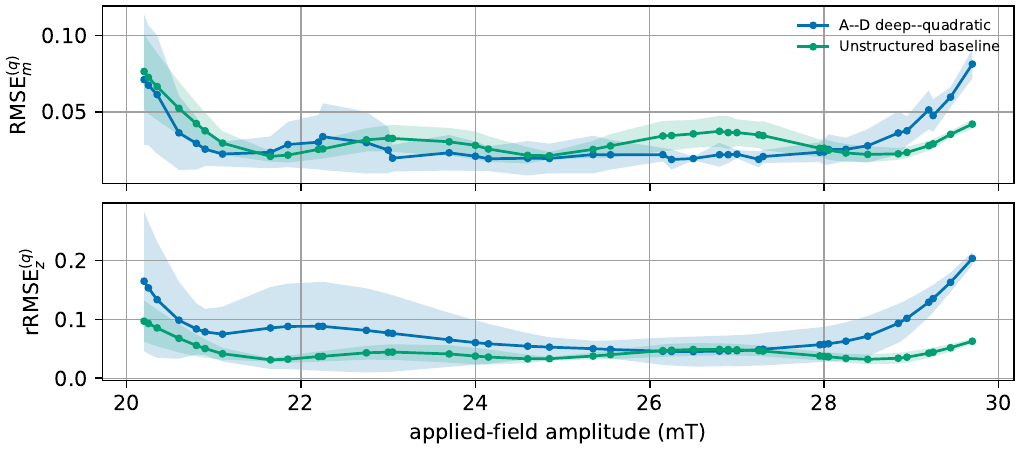}
\caption{{Field~1 full-rollout errors versus held-out amplitude: magnetization RMSE above and relative latent RMSE below. Lines show seed means and bands one sample standard deviation.}}\label{fig:amplitude-field1}\end{figure}
\begin{figure}[ht]\centering\includegraphics[width=.92\textwidth]{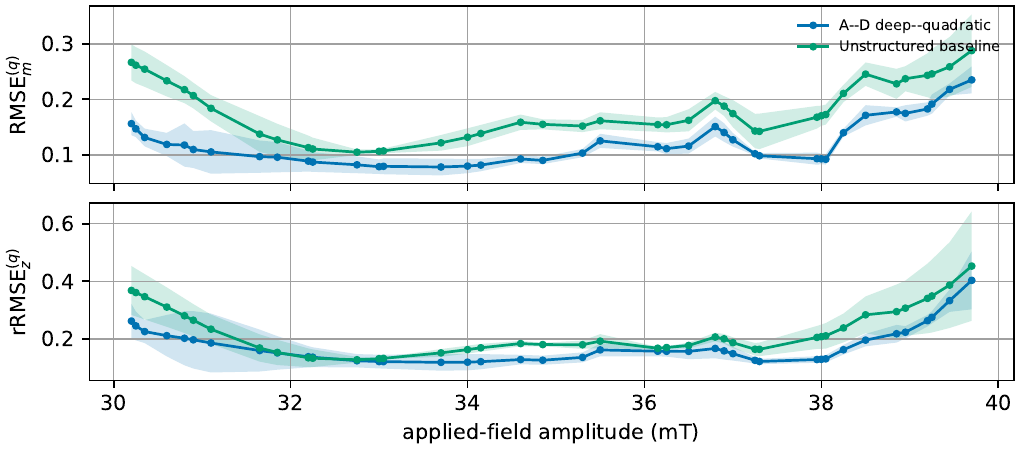}
\caption{{As \Cref{fig:amplitude-field1}, for field~2.}}\label{fig:amplitude-field2}\end{figure}
{Figures~\ref{fig:test-errors-field1} and \ref{fig:test-errors-field2} show errors pooled over all test trajectories at each saved time, then summarized across initializations. The A--D deep--quadratic model shows stronger relative latent-error growth toward the end of both rollouts. For field~1, this accompanies an increase in magnetization error beyond about \SI{0.8}{ns}, whereas the baseline error remains comparatively flat. For field~2, the structured model has lower magnetization error over most of the interval, despite its larger terminal latent error, and the two magnetization errors become comparable near \SI{1}{ns}. Each model learns its own coordinates and decoder, so relative latent error is not a common physical measure of prediction quality. The full-rank quadratic contribution may influence the late latent-error growth: it adds the term $2\kappa(\Jop-\Rop)\z$ to the vector field, including directions not controlled by the convolutional factor. The fixed value $\kappa=10^{-3}$ provides coercivity but need not give the best balance between boundedness and finite-time prediction accuracy. This is a possible explanation rather than an isolated effect, since no ablation of $\kappa$ was performed. Together with the spatial snapshots below, the figures distinguish error accumulation in learned coordinates from its effect on decoded magnetization.}
\begin{figure}[ht]\centering\includegraphics[width=.92\textwidth]{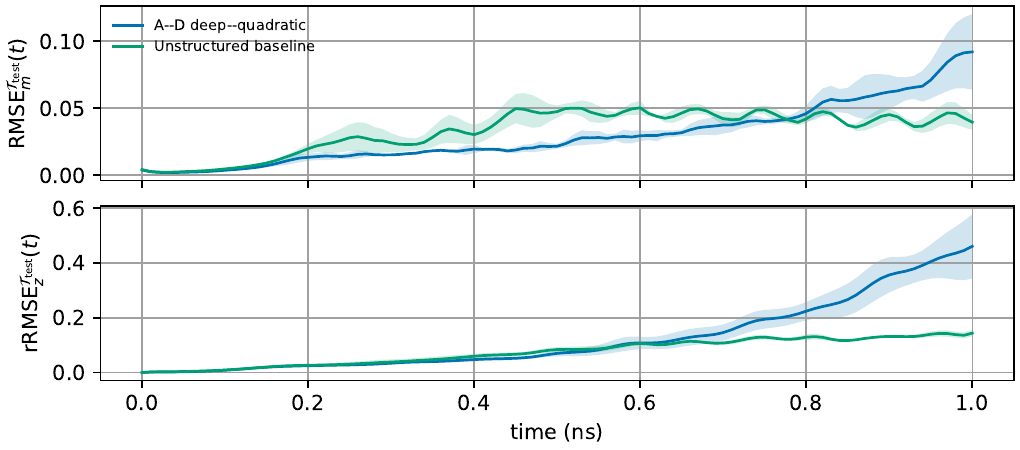}
\caption{{Time-resolved field~1 test errors: magnetization RMSE pooled over trajectories above and relative latent RMSE below. Lines and bands denote seed means and one sample standard deviation.}}\label{fig:test-errors-field1}\end{figure}
\begin{figure}[ht]\centering\includegraphics[width=.92\textwidth]{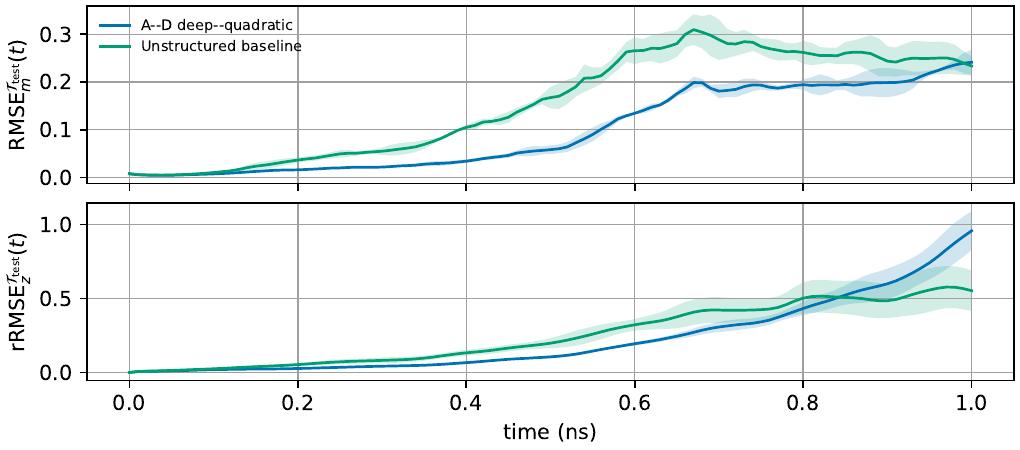}
\caption{{As \Cref{fig:test-errors-field1}, for field~2.}}\label{fig:test-errors-field2}\end{figure}

\subsection{Reference trajectories, {learned latent scalar}, and operator diagnostics}
{\Cref{fig:best-rollout-field1,fig:best-rollout-field2} show the corresponding full reference-case rollouts for the antisymmetric--dissipative deep--quadratic models. Each figure presents the mean-magnetization trajectory together with its rollout error, the change in learned latent scalar, and the Gibbs energy of the reference solution. The learned scalar does not qualitatively coincide with the Gibbs energy. Its decrease reflects the structural restriction on the latent flow rather than physical-energy recovery. The field-1 model captures the trajectory considerably more accurately than the field-2 model. Although the latter approximates the overall evolution, its much larger error does not support comparable accuracy in the two regimes. The datasets differ in both field direction and amplitude interval, so their comparison does not isolate either effect or identify a switching-related failure mechanism.}

\begin{figure}[ht]
\centering\includegraphics[width=.92\textwidth]{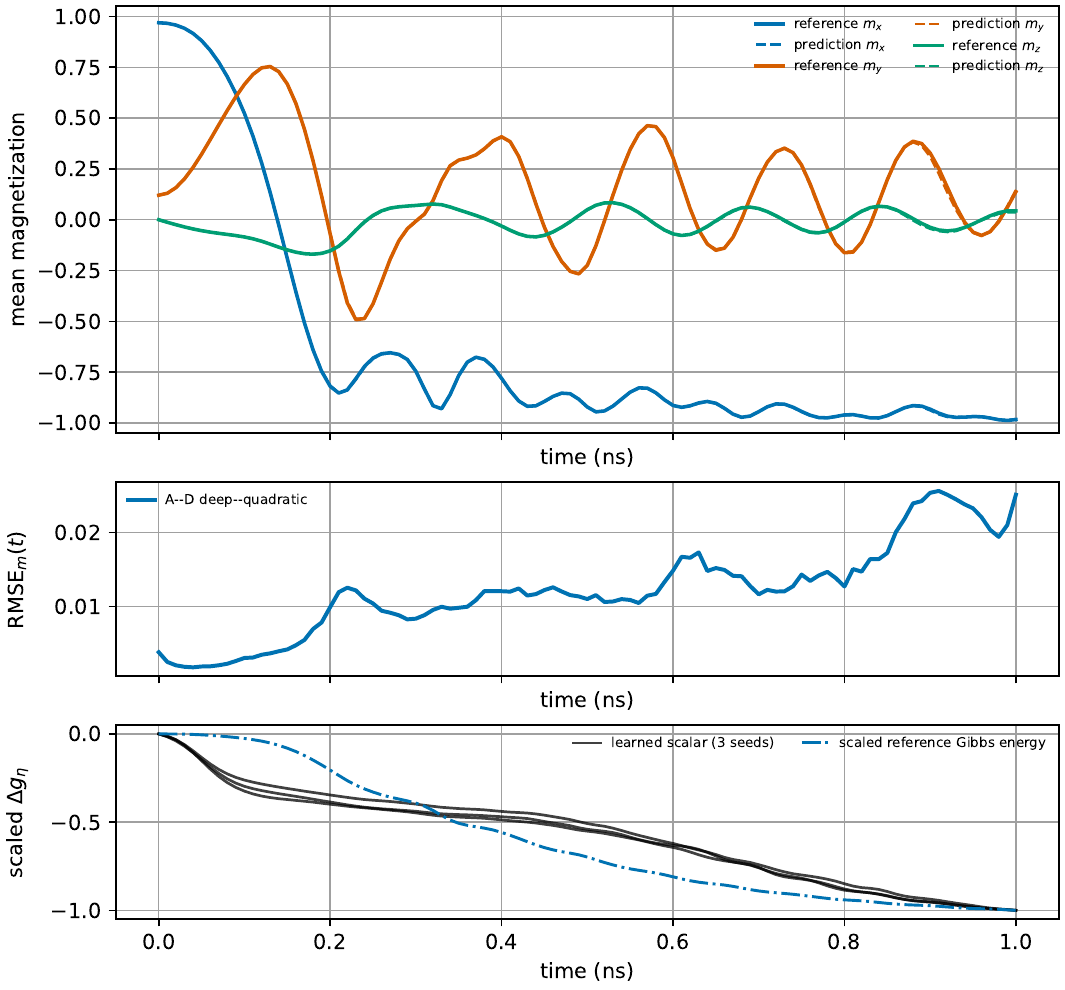}
\caption{{Validation-selected antisymmetric--dissipative deep--quadratic rollout at \SI{25}{mT}, $170^\circ$: spatially averaged magnetization and RMSE for the selected seed, independently normalized changes in the learned latent scalar for all three seeds, and scaled Gibbs energy of the reference solution. The rollout is evaluated on the $100\times25\times1$ grid with relative tolerance $10^{-3}$ and absolute tolerance $10^{-5}$. The learned latent scalar is not a physical energy.}}\label{fig:best-rollout-field1}
\end{figure}

\begin{figure}[ht]
\centering\includegraphics[width=.92\textwidth]{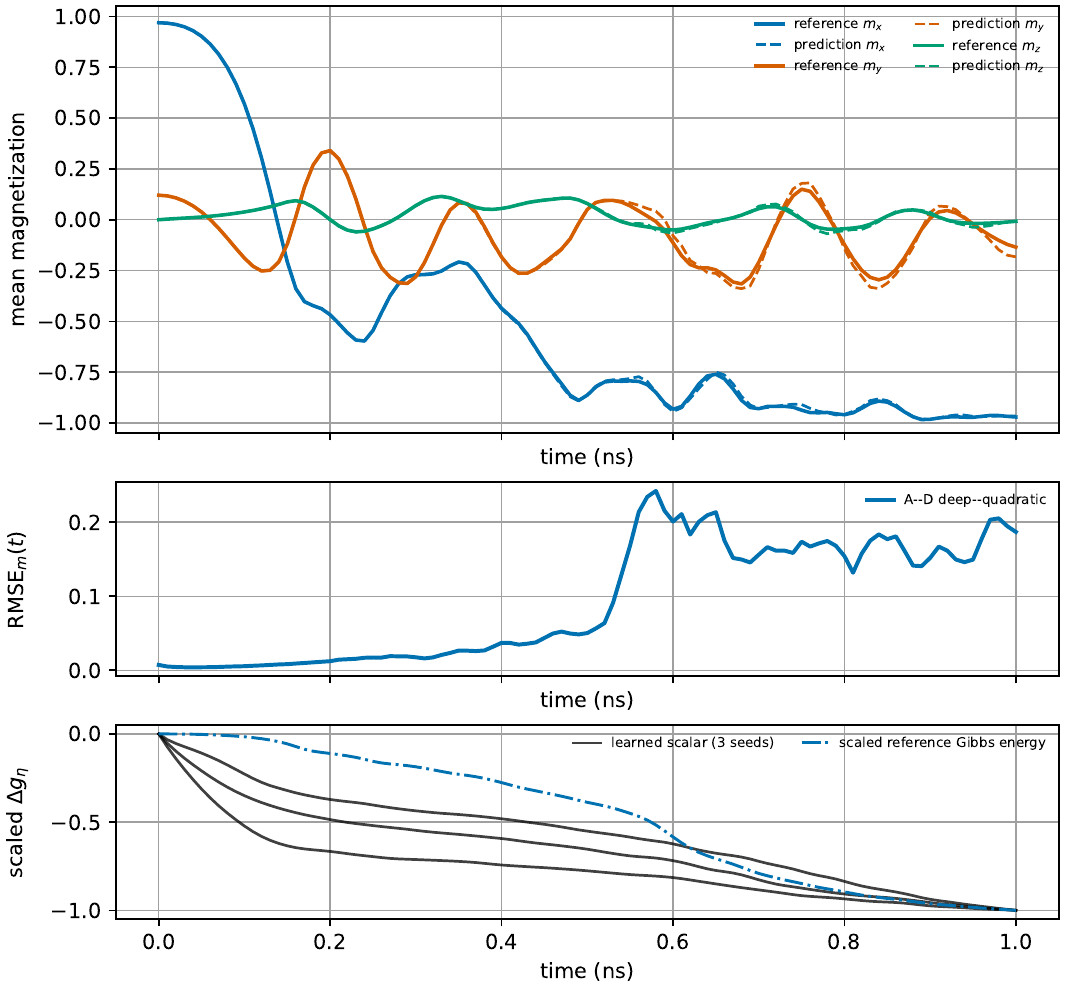}
\caption{{As \Cref{fig:best-rollout-field1}, at \SI{36}{mT}, $190^\circ$.}}\label{fig:best-rollout-field2}
\end{figure}

{The displayed learned-scalar traces decrease consistently with \eqref{eq:dissipation}. Each trace is shifted by its initial value and divided by its own maximum absolute change over the plotted interval. The reference Gibbs-energy change is scaled separately. No mean scalar across independently trained models is shown, since scalar offsets, scales, and latent coordinates are not identifiable.}

{The spatial structure of these rollouts is displayed in \Cref{fig:snapshots-field1,fig:snapshots-field2}. The angle maps compare the reference and decoded states at several times, while the error maps indicate vector differences. For field~2, localized vector errors become pronounced at \SI{0.6}{ns} and \SI{0.8}{ns}, when the reference develops more intricate spatial patterns. The prediction captures the broad evolution but misses some of this spatial detail. Reproducing the positions and timing of these features is more demanding than matching spatial means alone. The two figures use different error scales, with substantially larger errors for field~2. This identifies a spatial and temporal error pattern, but not a unique cause or a measured switching-time error.}

\begin{figure}[ht]
\centering\includegraphics[width=.92\textwidth]{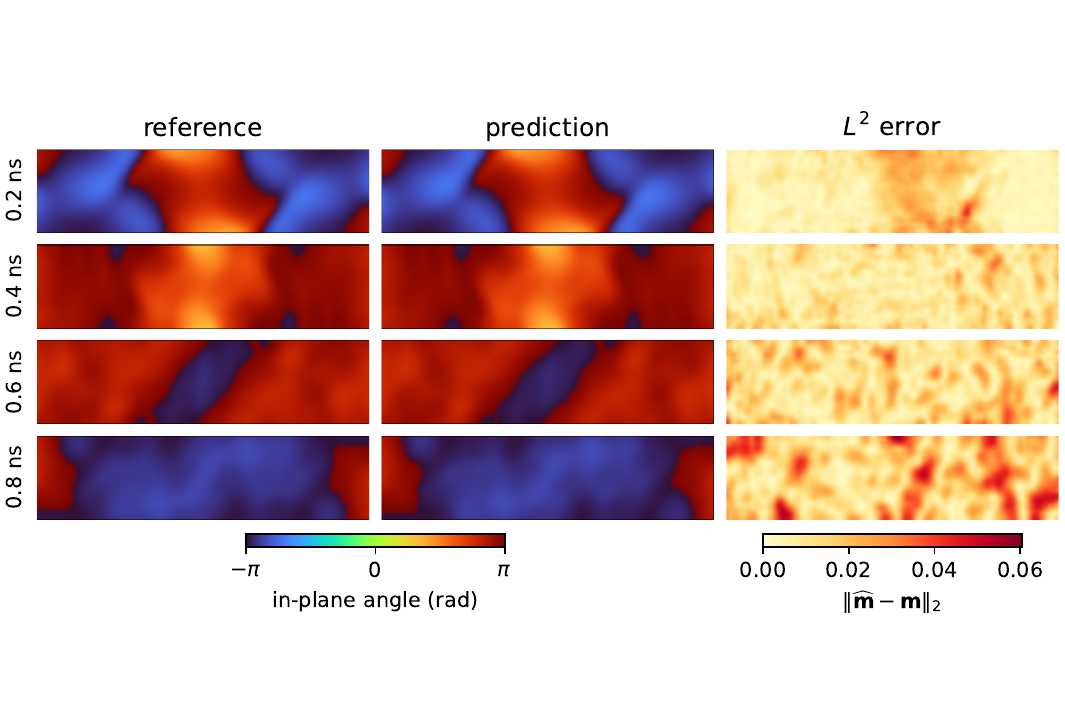}
\caption{{Reference and decoded in-plane magnetization angle for the validation-selected A--D deep--quadratic model at \SI{25}{mT}, $170^\circ$, on the $100\times25\times1$ grid. Both angle images use the common interval $[-\pi,\pi]$. The right column is the pointwise decoded-minus-reference vector error with one shared scale.}}\label{fig:snapshots-field1}
\end{figure}

\begin{figure}[ht]\centering\includegraphics[width=.92\textwidth]{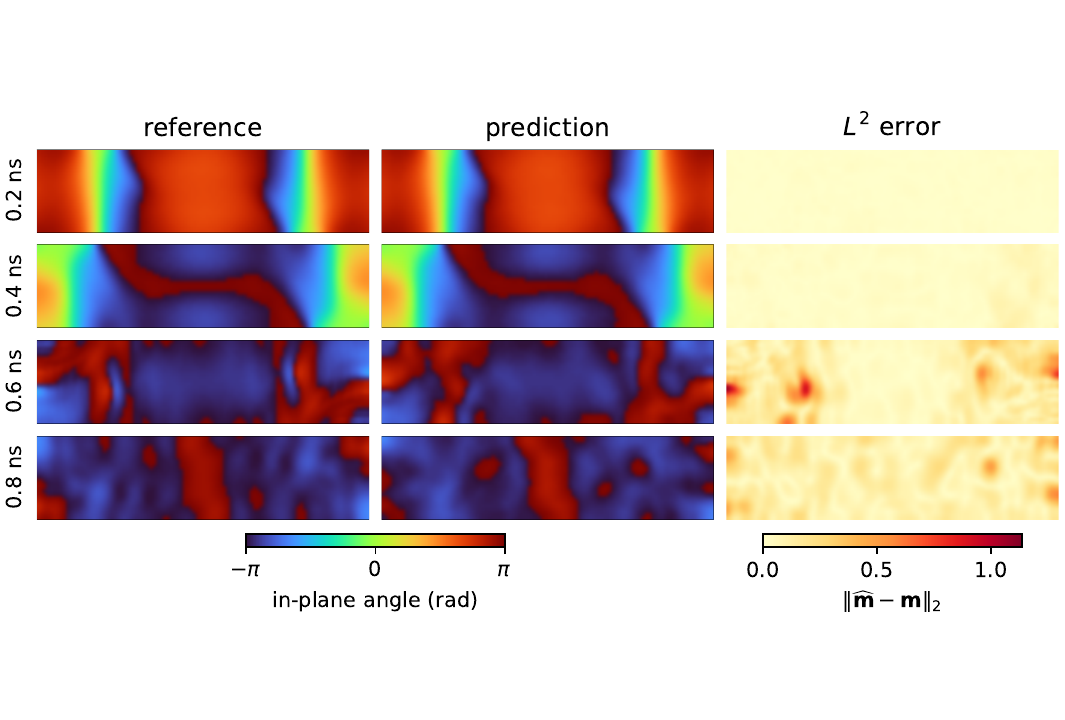}
\caption{{As \Cref{fig:snapshots-field1}, at \SI{36}{mT}, $190^\circ$.}}\label{fig:snapshots-field2}
\end{figure}

\begin{figure}[ht]
\centering\includegraphics[width=.82\textwidth]{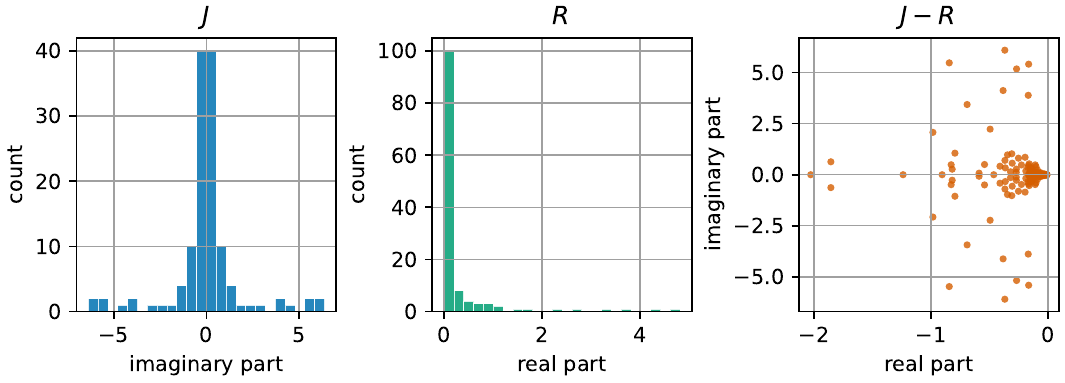}
\caption{{Operator eigenvalues for the validation-selected field~1 A--D deep--quadratic model. The first two panels histogram imaginary parts of the $\Jop$ spectrum and real parts of the $\Rop$ spectrum. The third shows $\Kop=\Jop-\Rop$. These diagnose operator structure, not relative dynamical contributions.}}\label{fig:operator-eigenvalues-field1}
\end{figure}
\begin{figure}[ht]
\centering\includegraphics[width=.82\textwidth]{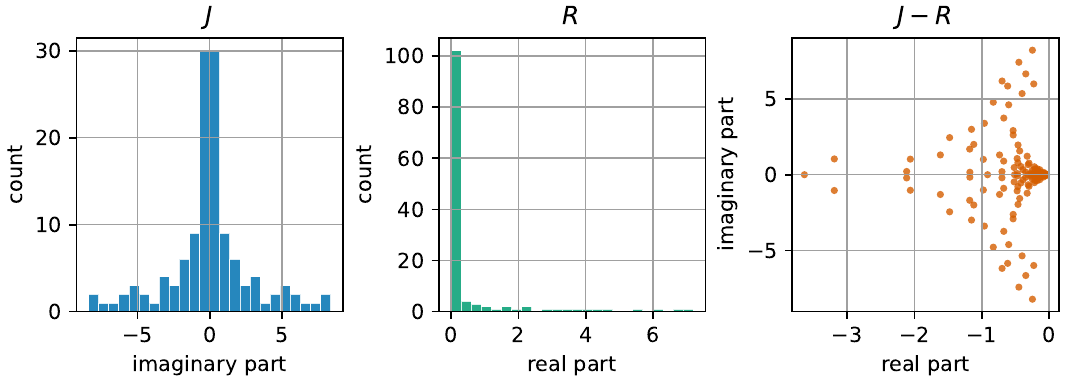}
\caption{{As \Cref{fig:operator-eigenvalues-field1}, for field~2.}}\label{fig:operator-eigenvalues-field2}
\end{figure}

{\Cref{fig:operator-eigenvalues-field1,fig:operator-eigenvalues-field2} report the learned channel-operator spectra for the two selected models. The spectra verify the imposed parameterizations, but their widths do not measure relative dynamical contributions. We therefore evaluate $\norm{\Jop\nabla_{\z}g_\eta}_2$ and $\norm{\Rop\nabla_{\z}g_\eta}_2$ along the learned rollouts initialized at the nominal reference states, at \SI{25}{mT} and \SI{36}{mT}. Figures~\ref{fig:contributions-field1} and \ref{fig:contributions-field2} show these norms. For each seed, we additionally sum both norms over all test trajectories and times and normalize each sum by their combined sum. Averaging these relative norm shares across seeds gives conservative contributions of $87.56\%\pm0.29\%$ for field~1 and $83.76\%\pm1.10\%$ for field~2. The conservative norm is larger throughout both displayed intervals. These quantities measure relative vector-field magnitudes along the evaluated rollouts, not the effect of removing either component or their importance for other trajectories.}
\begin{figure}[ht]\centering\includegraphics[width=.92\textwidth]{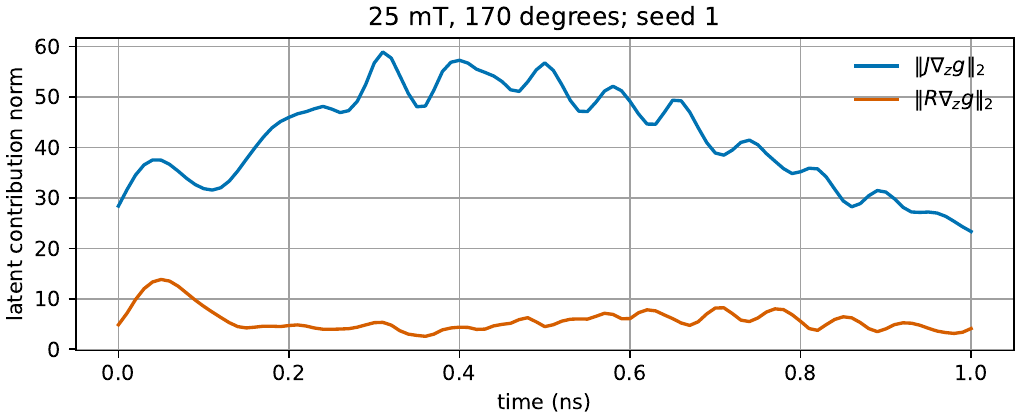}
\caption{{Conservative and dissipative contribution norms along the nominal field~1 rollout at \SI{25}{mT}, $170^\circ$, for the validation-selected model seed.}}\label{fig:contributions-field1}\end{figure}
\begin{figure}[ht]\centering\includegraphics[width=.92\textwidth]{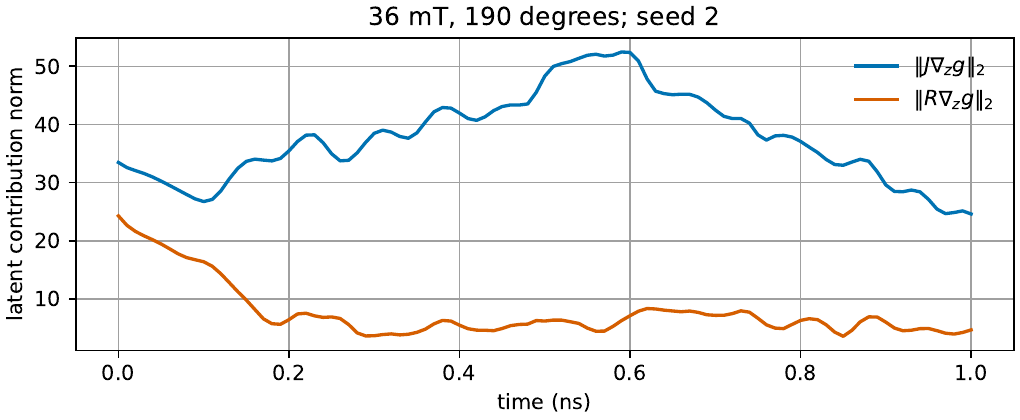}
\caption{{As \Cref{fig:contributions-field1}, for field~2 at \SI{36}{mT}, $190^\circ$.}}\label{fig:contributions-field2}\end{figure}

\begin{figure}[ht]
\centering\includegraphics[width=.92\textwidth]{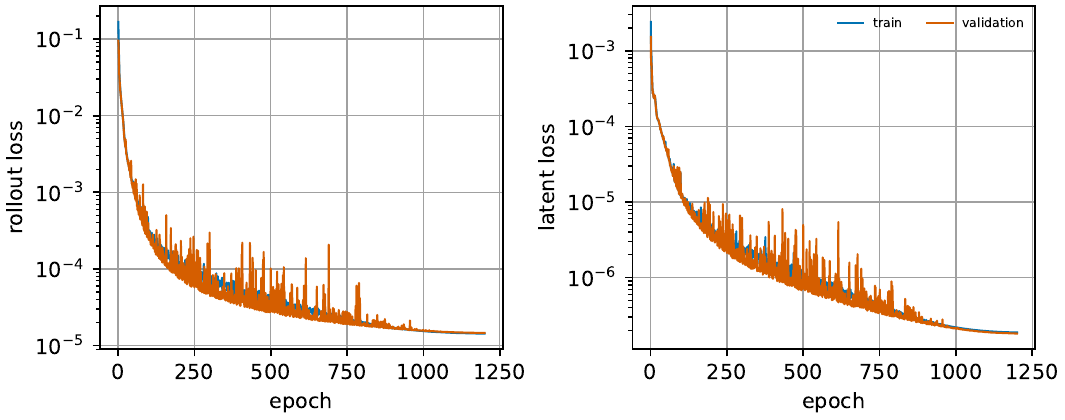}
\caption{{Training and validation decoded-rollout and latent losses for the validation-selected field~1 A--D deep--quadratic model.}}\label{fig:training-field1}
\end{figure}
\begin{figure}[ht]
\centering\includegraphics[width=.92\textwidth]{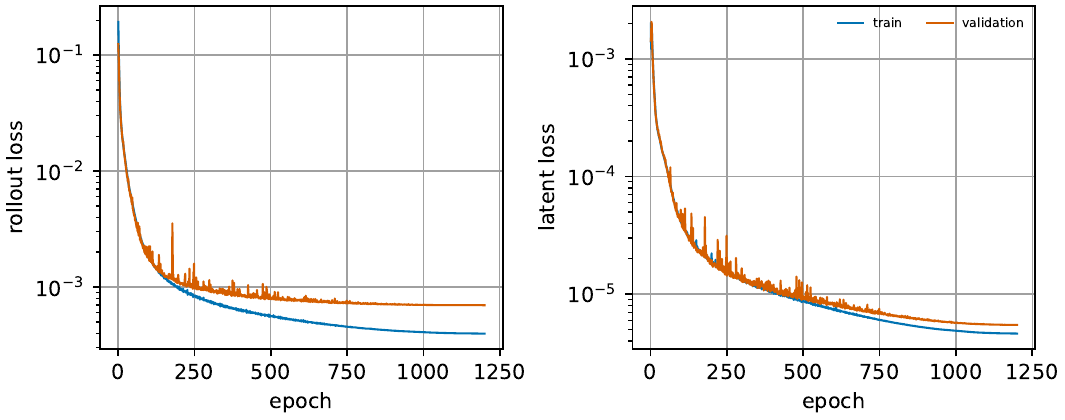}
\caption{{As \Cref{fig:training-field1}, for field~2.}}\label{fig:training-field2}
\end{figure}
{The optimization histories in \Cref{fig:training-field1,fig:training-field2} show the training and validation contributions used to fit the two selected models. For field~2, both training and validation losses remain higher than for field~1. This is consistent with a more difficult trajectory family for the present model and training procedure, but does not distinguish insufficient representation capacity from optimization or dynamical-prediction errors.}

\subsection{{Equilibrium retention beyond the training horizon}}
{We extend selected A--D deep--quadratic and unstructured baseline models to \SI{5}{ns} at the two nominal reference fields, without retraining or intermediate re-encoding. Both encode the same freshly relaxed physical reference state. Figures~\ref{fig:extrapolation-field1} and \ref{fig:extrapolation-field2} compare spatial means and magnetization RMSE against a direct high-accuracy reference solve. The marker at \SI{1}{ns} separates the training horizon from temporal extrapolation.}

{These are selected single-model illustrations. The three-initialization statistics above describe the observed-horizon test comparison, not a long-time success frequency.}

{For field~1, both models initially follow the relaxation, but the baseline subsequently drifts and departs further from the stationary reference. For field~2, baseline excursions grow strongly beyond the training horizon, whereas the structured trajectory remains bounded and qualitatively stationary, with a nonzero offset from the reference. Thus, competitive test accuracy on \SIrange{0}{1}{ns} does not imply comparable equilibrium retention. The revised quadratic and deep--quadratic models satisfy Proposition~\ref{prop:boundedness} and Corollary~\ref{cor:lasalle} independently of the rank of the learned convolutional factor, since their remaining scalar contribution is nonnegative. Their scalar-decrease property and compact sublevel sets provide a mechanism for constraining such departures. They do not guarantee convergence to the correct physical equilibrium, contraction between nearby solutions, or stability of every numerical rollout. The quadratic term supplies an affine component of the latent dynamics, related to the stabilizing inductive biases considered in long-time neural-ODE forecasting \cite{linot2023stabilized}.}

{For the three main A--D deep--quadratic runs, the learned $\beta$ values are $(4.402,5.159,4.072)\times10^{-5}$ for field~1 and $(1.253,1.299,1.277)\times10^{-5}$ for field~2. They keep the dissipation positive definite, but the small lower spectral bound does not imply rapid relaxation.}

{At \SI{5}{ns}, the scalar-gradient norms of the displayed structured trajectories are $0.543$ for field~1 and $0.244$ for field~2, down from $8.65$ and $7.55$ initially. These are finite-time diagnostics in learned latent coordinates, not physical residuals or convergence tolerances. The weak lower dissipation bound permits slow relaxation, consistent with a gradient still appreciably above zero at \SI{5}{ns}. Small eigenvalues alone do not determine the observed rate. The terminal latent vector-field norms are $0.588$ and $0.516$ for the structured model, versus $10.1$ and $1.14\times10^4$ for the displayed baselines. The asymptotic conclusion $\norm{\nabla g}_2\to0$ concerns the exact coercive autonomous system, not a guarantee that any fixed finite horizon reaches an equilibrium.

We do not claim universal superiority of the structured model over unrestricted neural ODEs or other micromagnetic surrogates. The present results demonstrate competitive accuracy within the observed horizon and illustrate the potential benefit of structural boundedness for longer equilibrium-seeking trajectories.
The baseline model performs very well within the training horizon and should certainly be considered if applicable for the task at hand. However, the latent scalar models can provide more structure and stronger mathematical guarantees for equilibrium-seeking surrogate tasks. Application to micromagnetic energy minimization would additionally require checking decoded states against the physical energy and equilibrium conditions.}

\begin{figure}[ht]\centering\includegraphics[width=.92\textwidth]{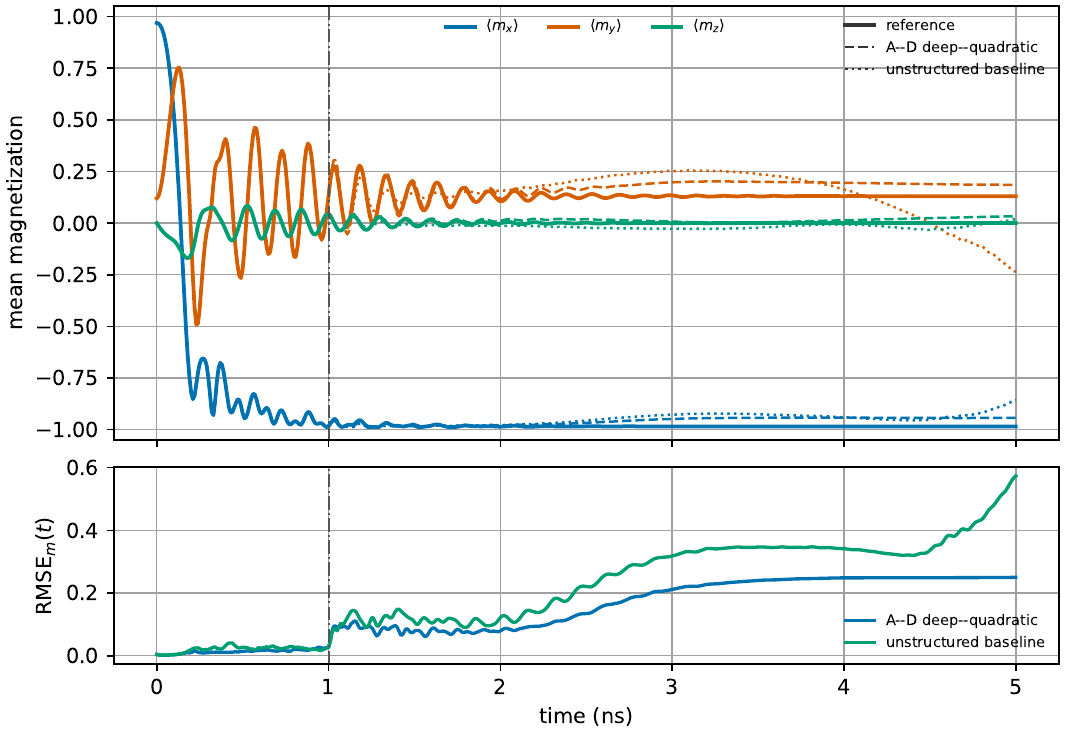}
\caption{{\SI{5}{ns} rollout at \SI{25}{mT}, $170^\circ$. The upper panel shows reference, A--D deep--quadratic, and unstructured-baseline spatial means. Colors denote components and line styles denote models. The lower panel shows magnetization RMSE. Both models start from the same reference state. The vertical marker is the \SI{1}{ns} training horizon.}}\label{fig:extrapolation-field1}\end{figure}
\begin{figure}[ht]\centering\includegraphics[width=.92\textwidth]{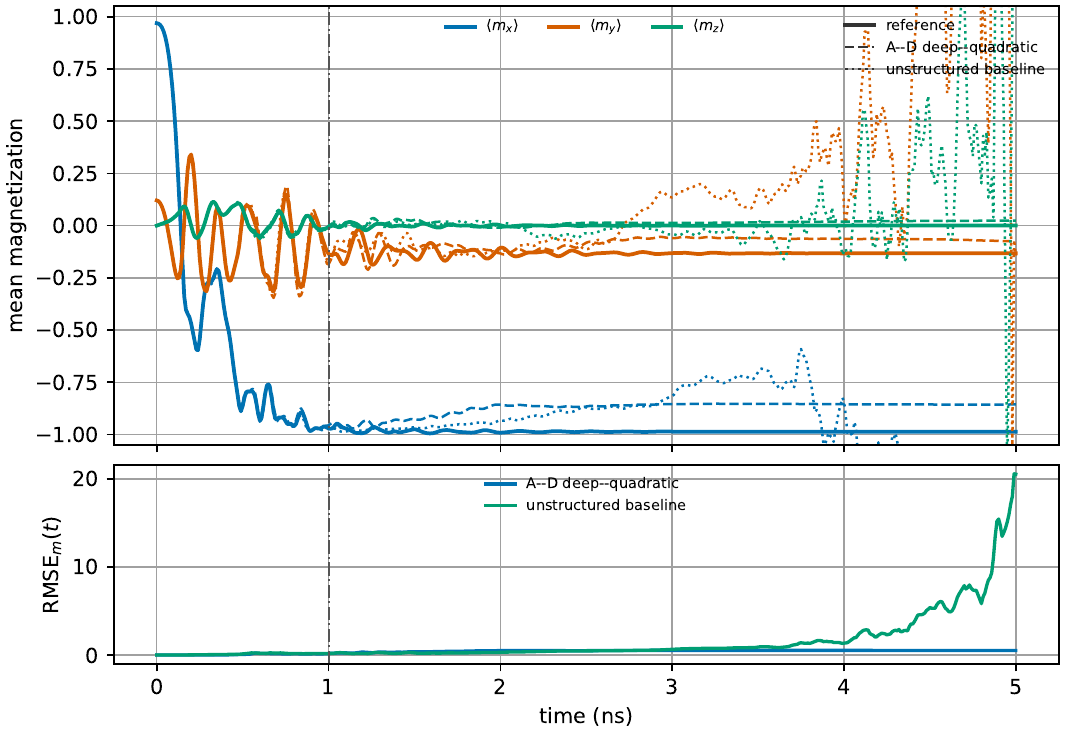}
\caption{{As \Cref{fig:extrapolation-field1}, at \SI{36}{mT}, $190^\circ$.}}\label{fig:extrapolation-field2}\end{figure}

\clearpage

\section{Limitations and prospective applications}
{The experiments vary amplitude within two fixed-direction datasets while keeping geometry, discretization, and material parameters fixed. They do not establish accuracy outside the sampled intervals, for unseen directions, or for other materials. The present benchmark therefore supports conclusions about the tested trajectory families rather than broad parameter-space generalization.}

{A larger training set could improve coverage of the amplitude-dependent trajectory family and reduce sampling-related generalization errors. It could also change the relative performance of the models, which need not benefit equally from additional data. However, more densely sampled amplitudes from the same initial state do not introduce new initial conditions or field directions, and more data alone need not resolve representation or optimization limitations. The observed generalization gap does not establish that increasing the dataset size alone would resolve the error. Without a training-set-size study, neither the magnitude of an accuracy improvement nor preservation of the observed model ranking can be inferred.}

{Learning trajectories over larger conditional parameter spaces can require substantially more data. A promising idea from energy-based models is to employ a bounded replay buffer of positive supervised PDE transitions. An offline active reference-solver could be advanced by one output interval, with the resulting one-step transitions or short windows stored in the buffer. An instance would be replaced after reaching a predefined termination criterion. This stochastic approach could allow the training over large conditional spaces.}

{To test this, a natural extension would include both field magnitude and field angle as conditioning parameters. Further extensions should vary damping, material parameters, and geometry. The convolutional encoder and decoder can in principle be instantiated for different compatible spatial dimensions, subject to their kernel, stride, and padding choices. They are not invariant with respect to spatial discretization: refining the grid on a fixed physical domain changes the physical extent of a convolutional kernel and the sampling of magnetization structures. Thus, using the same weights at a higher resolution is not automatically a consistent representation of the same physical problem. Predictive generalization across resolutions, domain sizes, or shapes would require dedicated training and evaluation across spatial configurations.}

{Standard Problem~4 is a well-established micromagnetic benchmark \cite{nist_sp4}, and the present $100\times25\times1$ discretization has also been used in earlier surrogate studies \cite{exl2021prediction,schaffer2021machine}. Our timing comparison concerns online evaluation on this grid, which the explicit adaptive reference solver integrates efficiently. Changing the grid changes the full-order degrees of freedom and field-evaluation cost. Refinement can also increase exchange-induced stiffness and restrict the adaptive time steps. A speed comparison on another grid must therefore be measured again. More restrictive full-order regimes motivate further tests but do not support a quantitative speedup claim beyond the present grid and solver configuration.}

{Application-specific accuracy must likewise be assessed for the relevant quantity of interest. Switching times, terminal states, and spatially averaged observables can impose different tolerances, and a small global trajectory RMSE need not imply an accurate switching event. The present benchmark does not establish an acceptable error threshold for device design, parameter optimization, or inverse design. Comparing against such tolerances is outside the scope of this conceptual study and requires application-specific datasets, observables, and validation criteria. We defer this assessment to future research. In particular, the larger field-2 error calls for further validation before use in tasks requiring accurate reversal dynamics.}

{Low-rank channel-operator parameterizations could drastically decrease the required parameter count for the latent vector field model. We plan to investigate these together with state-dependent operators. For example, with state-dependent factors $U,V,W\in\R^{n_c\times r}$, where $n_c$ is the channel width and $r\ll n_c$, the choices
\[
 \Jop=UV^{\mathsf T}-VU^{\mathsf T},\qquad
 \Rop=WW^{\mathsf T}+\beta I,\quad\beta>0,
\]
retain antisymmetry and positive definiteness. A compact network would predict $\mathcal O(n_c r)$ factor entries rather than $\mathcal O(n_c^2)$ dense matrix entries, providing a route to state dependence with a limited parameter count. The $\beta I$ term keeps $\Rop$ full rank despite its low-rank learned correction. The presented spectra motivate investigating such factorizations, but do not establish that rank restrictions preserve the learned trajectories. This extension and the choice of factor rank are deferred to future work.}

In integrated sensing, a compact latent rollout could be useful if task-specific accuracy, latency, and robustness are demonstrated. When trajectories share an initial state, it only needs to be encoded once and only a terminal state could be decoded if intermediate fields are unnecessary. For observables, such as mean magnetization, a lightweight latent-to-observable head could replace the decoder completely. This could potentially result in a very lightweight model class.

{The construction is also potentially useful beyond micromagnetics, in learned continuous-time representations where bounded trajectories and equilibrium-seeking behavior matter. A coercive scalar combined with antisymmetric and dissipative components offers a structural inductive bias for neural ODEs in machine learning. The present results motivate this direction.}

\section{Conclusion}
{We coupled an autoencoder to a structured latent ODE generated by a learned scalar through antisymmetric and positive-definite dissipative operators. The scalar is not physical Gibbs energy. A full-rank coercive quadratic term ensures bounded exact latent trajectories and, under the stated assumptions, approach to the scalar's critical set. On two SP4 field-amplitude datasets, three-initialization comparisons show that the A--D deep--quadratic and generic baseline models can achieve comparable observed-horizon accuracy, particularly for field~1, while dissipative-only models are much less accurate on uninterrupted rollouts. The main distinction is therefore not universal short-horizon superiority: the selected \SI{5}{ns} tests show stronger equilibrium retention for the structured models, whereas baseline drift or growing excursions occur after the training horizon. Neither correct physical equilibria nor general long-time prediction accuracy follows from the latent construction and must be learned from data. The results motivate structured neural ODEs for other stability-sensitive learning tasks, subject to further validation, and future studies over richer fields, materials, geometries, and discretizations.}\\

\textbf{Data.} The source code used to produce the results is publicly available on Zenodo \cite{schaffer_2026_22281254}.

\newpage

\section*{Acknowledgment} 
This research was funded in whole or in part by the Austrian Science Fund (FWF) [10.55776/PAT7615923, 10.55776/P35413]. For the purpose of Open Access, the authors have applied a CC BY public copyright license to any Author Accepted Manuscript (AAM) version arising from this submission. The computations were partly achieved by using the Austrian Scientific Cluster (ASC) via the funded projects No. 71140, 71952 and 72862.


\end{document}